%% file: main_arxiv.tex
\documentclass{article}
\usepackage[preprint]{icml2026}
\usepackage[round]{natbib}
\usepackage{algorithm}
\usepackage{algorithmic}
\usepackage{graphicx}
\usepackage{amsmath, amsfonts, dsfont}
\usepackage{amsthm, mathabx}
\usepackage{subcaption}
\usepackage{xcolor}
\usepackage{hyperref}
\usepackage{makecell}
\usepackage{wrapfig}
\usepackage{enumitem}
\usepackage{booktabs}
\usepackage{multirow}
\usepackage{fancyhdr}
\usepackage{comment}
\usepackage{bbm}
\usepackage{caption}
\usepackage{siunitx}
\usepackage{tikz}
\usetikzlibrary{positioning,fit,arrows.meta,calc}

\newtheorem{theorem}{Theorem}
\newtheorem{prop}{Proposition}

\DeclareMathOperator*{\argmax}{arg\,max}

\newcolumntype{C}[1]{>{\centering\arraybackslash}p{#1}}

\setlist[enumerate]{wide=0pt, leftmargin=16pt, labelwidth=10pt, align=left}
\numberwithin{equation}{section}

\icmltitlerunning{Monotone Optimisation with Learned Projections}

\newcommand{\algtitle}[2]{%
  \hrule\vspace{0.6ex}%
  \refstepcounter{algorithm}\label{#1}%
  \parbox[t][2\baselineskip][t]{\linewidth}{\small\raggedright
    \textit{Algorithm~\thealgorithm.}~#2\strut}%
  \par\vspace{0.4ex}\hrule}

\begin{document}
\twocolumn[
\icmltitle{Monotone Optimisation with Learned Projections}

\begin{icmlauthorlist}
    \icmlauthor{Ahmed Rashwan}{bath}
    \icmlauthor{Keith Briggs}{bt}
    \icmlauthor{Chris Budd}{bath}
    \icmlauthor{Lisa Kreusser}{bath,mon}
\end{icmlauthorlist}

\icmlaffiliation{bath}{University of Bath}
\icmlaffiliation{bt}{BT Research}
\icmlaffiliation{mon}{Monumo}

\icmlcorrespondingauthor{Ahmed Rashwan}{ar3009@bath.ac.uk}
\vskip 0.3in
]
\printAffiliationsAndNotice{}

\begin{abstract}
Optimisation algorithms typically assume oracle access to objectives and constraints, yet in many applications these functions are available only through data.
To recover them, a common approach learns a surrogate for each function and optimises over it, but this can be unnecessarily expensive and may fail to preserve the structure required by the optimisation algorithm.
We study this problem for monotone optimisation, a class of non-convex problems which can be solved globally using the \emph{Polyblock Outer Approximation} (POA) algorithm.
Our key observation is that POA does not require the constraint functions themselves, but only a scalar quantity we call the \emph{radial inverse}.
Learning this quantity directly replaces POA's inner bisection loop over a learned constraint with a single radial inverse evaluation.
We characterise exactly which maps arise as radial inverses of monotone constraints, and relate their approximation error to perturbations of the feasible set and optimal objective value.
Building on these results, we introduce \emph{Homogeneous--Monotone Radial Inverse (HM-RI) networks}, structured min--max models that enforce the monotonicity and positive-homogeneity of radial inverses by construction.
Across synthetic monotone
optimisation problems and a cellular transmit-power
application, HM-RI improves solution quality over direct constraint surrogates and local optimisation baselines by up to 58.6 percentage points relative to the \textsc{Oracle} reference objective, while achieving a $3.0$--$4.7\times$ speed-up over bisection-based POA with learned constraint surrogates, with the largest solution-quality gains occurring in data-limited and piecewise-constant settings.

\end{abstract}

\section{Introduction}

Many optimisation problems in modern machine-learning pipelines involve objectives or constraints that are not analytically available and must instead be inferred from data \citep{optim_learned_constraints}.
Examples include learned surrogate objectives in model-free optimisation \citep{learned_objective_optim, learned_optimisation_bs_sleep} and regularisers learned through bilevel training for inverse problems \citep{learned_regularisation_survey, depp_eq_regulariser}.
A common strategy is to learn a surrogate of the unknown function and then optimise that surrogate.
Scalable local methods can be used for this purpose \citep{grad-based_learned_optim, learned_optim_sampling}, but provide no guarantees of global optimality, while exact formulations using mixed-integer programming can be prohibitively expensive \citep{Feedforward_MILP_optim}.
More fundamentally, learning a complete surrogate may be unnecessary: an optimisation algorithm often interacts with an unknown function only through a small set of algorithm-specific operations, while requiring those operations to satisfy particular structural properties.

A particularly useful structure is monotonicity, which arises, for example, in resource allocation and communication systems, as well as simulation-based optimisation \citep{monotone_opt_comms_book, computer_simulation, mono_RA_optim}.
Although such problems can remain highly non-convex, monotonicity enables
specialised global optimisation methods
\citep{low-rank_nonconvex_book, monotone_opt_tuy, bjornson2012robust}.
Despite its practical relevance, data-driven monotone optimisation remains underexplored.

The primary method for global monotone optimisation is the \emph{Polyblock Outer Approximation} (POA) algorithm \citep{monotonic_opt_tuy2}.
POA exploits monotonicity to iteratively refine a rectangular (\emph{polyblock}) approximation of the feasible region, providing global
optimality guarantees.
One of its core subroutines is a monotone projection that is evaluated repeatedly throughout the solve.
In many applications, monotonicity is known from domain knowledge even when constraint functions are unavailable or expensive to evaluate, making POA a natural candidate for data-driven global optimisation.
This motivates an algorithm-aware learning approach
\citep{learned_proximal_operator}: rather than learning the constraint as an end in itself, we ask what quantity POA actually queries.

A key observation here is that POA does not require explicit access to the upper-bound constraint functions themselves.
Instead, its projection step amounts to a one-dimensional inversion along a ray, computed by bisection.
Because this projection is repeated throughout POA,
it can require many constraint or surrogate evaluations and can become a substantial computational bottleneck.
We identify the corresponding \emph{radial inverse} as the solver primitive required by POA and
learn it directly.
This avoids first learning a constraint and then repeatedly inverting the resulting surrogate, providing a tighter interface between learning and optimisation and reducing modelling complexity.

Learning the radial inverse directly requires preserving its underlying structure, part of which is monotonic.
While classical monotone interpolation scales poorly
\citep{beliakov_monotonicity},  neural architectures can enforce monotonicity by construction, including min--max networks \citep{sill_monotonic,min-max_nets}, Deep Lattice Networks \citep{deep_lattice_nets}, and constrained monotone architectures \citep{constraints_mono_nets}.
Smooth min--max networks \citep{igel} avoid the sparse-gradient behaviour of hard min--max architectures and provide a natural starting point for our setting, where we modify the smoothing to also preserve degree-one positive homogeneity.

\subsection{Contributions}

\begin{itemize}

\item We develop an \emph{algorithm-aware learning framework} that extends POA to data-driven monotone optimisation, allowing learned constraints to be integrated while preserving the structural conditions required for POA to globally optimise the learned problem.

\item We introduce the \emph{radial inverse} as a constraint representation and learning target for POA, provide a theoretical characterisation of valid radial inverse mappings and their correspondence to underlying monotone constraint sets, and relate approximation error to perturbations of the feasible set and optimal objective value.

\item We propose \emph{Homogeneous--Monotone Radial Inverse (HM-RI)} networks
together with a power-mean smoothing that preserves
monotonicity and positive homogeneity exactly while enabling gradient-based training.
HM-RI satisfies the radial-inverse structural conditions by construction.
We establish universal approximation for continuous radial inverses.

\item HM-RI is \(3.0\)--\(4.7\times\) faster than bisection-based POA with learned constraint surrogates while improving solution quality, and outperforms local optimisation baselines by up to 58.6 percentage points relative to the \textsc{Oracle} reference objective. Structural ablations further show that enforcing monotonicity and positive homogeneity is an important contributor to the improvements in solution quality, beyond the radial-inverse formulation alone.

\end{itemize}

\section{Preliminaries}\label{sec:prelim}

\subsection{Monotone optimisation}\label{ssec:mono_opt_prelim}

\paragraph{Notation.} For $n \in \mathbb{N}$, let $[n] := \{1,\dots,n\}$. We write $\mathbb{R}_+ := [0,\infty)$ and denote by $\overline{\mathbb{R}} := [-\infty,+\infty]$ the extended real line, ordered in the usual way.
For vectors $x,y \in \mathbb{R}^n$, we use the partial product order: $x \le y$ (resp.\ $x < y$) if $x_i \le y_i$ (resp.\ $x_i < y_i$) for all $i \in [n]$.
Given $x \le y$, define the (closed) axis-aligned box
$[x,y] := \prod_{i=1}^n [x_i,y_i]$,
with open boxes defined analogously.

A function $f:\mathbb{R}_+^n\to\mathbb{R}$ is \emph{(strictly) increasing}
if $f(x)\le f(y)$ (resp.\ $f(x)<f(y)$) whenever $x\le y$
(resp.\ $x<y$). It is \emph{(strictly) decreasing} if $-f$ is (strictly) increasing. We say that $f$ is \emph{(strictly) monotone} if it is either (strictly) increasing or (strictly) decreasing.
We use the same terminology for extended-real-valued functions
$g:\mathbb{R}_+^n\to\overline{\mathbb{R}}$.

\paragraph{Rectangles and polyblocks.}
For any $x \in \mathbb{R}_+^n$, define the associated \emph{rectangle}
$P_x := [0,x] := \prod_{i=1}^n [0,x_i]$.
For a set $S \subseteq \mathbb{R}_+^n$, define its induced \emph{rectangular hull}
$P_S := \bigcup_{s \in S} P_s$.
Intuitively, $P_S$ is the set of all points dominated by some $s \in S$.
A \emph{polyblock} is a rectangular hull generated by a finite set $V$, i.e.\ a set of the form $P_V$ for finite $V \subset \mathbb{R}_+^n$.
A set $S \subset \mathbb{R}^n_+$ is said to be \emph{normal} if $P_S = S$ or, equivalently, if $x \in S$ and $0 \le y \le x$ implies $y \in S$. By monotonicity of the constraint functions, the sets $\mathcal{G}$ and $\mathcal{H}^c$ are normal.

\paragraph{Monotone projection.}
For a compact normal set $S \subseteq \mathbb{R}_+^n$ and $x \in \mathbb{R}_+^n$, define the \emph{monotone projection} of $x$ onto $S$ by $\pi_S(x) := \alpha_S(x)\,x,$ where $\alpha_S(x) := \sup\{ \alpha \in \mathbb{R}_+ : \alpha x \in S \}.$
Equivalently, $\pi_S(x)$ is the largest point on the ray $\{rx : r \ge 0\}$ that lies in $S$. At the origin we set $\pi_S(0) := 0$.

\begin{figure}[t]
    \centering
    \includegraphics[width=0.8\linewidth]{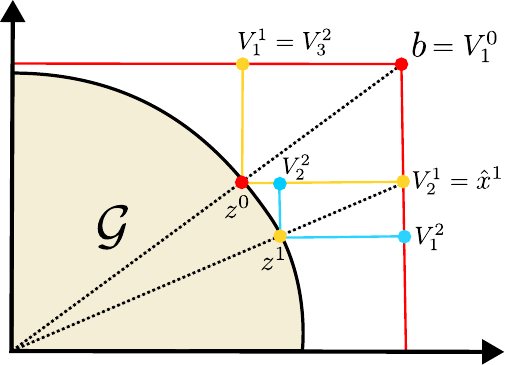}
    \caption{Illustration of the construction of $V^0=\{b\}$,
    $V^1=\{V^1_1,V^1_2\}$, and
    $V^2=\{V^2_1,V^2_2,V^2_3\}$ of POA in 2D with $m_h=0$,
    i.e., $\mathcal{H}=\mathbb{R}^2_+$.
    Points are colour-coded by iteration. Vertices are indexed by subscripts.}
    \label{fig:poa}
\end{figure}

\paragraph{Monotonic Problems.} We study monotone optimisation problems of the form
\begin{align}
\label{eq:mono_prob}
\begin{split}
    \max_{x \in \mathbb{R}_+^n} \quad & f(x) \\
    \text{s.t.} \quad & g_i(x) \le u_i,\quad i \in [m_g], \\
    & h_j(x) \ge l_j,\quad j \in [m_h],
\end{split}
\end{align}
where the objective $f:\mathbb{R}_+^n\to\mathbb{R}$ and the
constraint functions
$g_i,h_j:\mathbb{R}_+^n\to\overline{\mathbb{R}}$
are increasing, and all thresholds $u_i,l_j$ are finite.
Let
\begin{align}\label{eq:defgh}
\begin{split}
    \mathcal{G} &:= \{x \in \mathbb{R}_+^n : g_i(x) \le u_i,\ i \in [m_g]\},
\\
    \mathcal{H} &:= \{x \in \mathbb{R}_+^n : h_j(x) \ge l_j,\ j \in [m_h]\},
\end{split}
\end{align}
so that the feasible set is $\mathcal{F} := \mathcal{G} \cap \mathcal{H}$, which we assume is non-empty and that $\mathcal{G}$ is compact (e.g., by restricting $x$ to a bounded box $[0,b]\subset\mathbb{R}^n_+$).

We allow extended-real-valued constraints only to avoid an additional finiteness assumption in the radial-inverse characterisation of Section~\ref{sec:rad_inv}.
The assumption has no impact on the optimisation process, which relies only on the fact that $f, g_i, h_j$ preserve the partial order defined on $\mathbb{R}^n_+$, rather than on the particular values those functions take.

\paragraph{Polyblock Outer Approximation (POA).}
The Polyblock Outer Approximation (POA) algorithm solves \eqref{eq:mono_prob} by constructing a nested sequence of polyblock supersets of the feasible region
$P_{V^0} \supseteq P_{V^1} \supseteq \cdots \supseteq \mathcal{F}$,
where each $V^t \subset \mathbb{R}_+^n$ is a finite set of vertices.
Since $f$ is increasing, its maximum over a polyblock is attained at one of the vertices:
\begin{align}\label{eq:maximiser}
\hat{x}^t \in \arg\max_{x \in V^t} f(x)
\subset
\arg\max_{x \in P_{V^t}} f(x).
\end{align}
Thus $f(\hat{x}^t)$ provides an upper bound on the optimal value $f^\star := \max_{x \in \mathcal{F}} f(x)$. POA refines $V^t$ to iteratively improve this upper bound, while also searching for feasible candidates.

Starting from a dominant vertex $b$ with $\mathcal F\subseteq P_b$ and $V^0:=\{b\}$, POA iteratively tightens the sequence of polyblock approximations $P_{V^t}\supseteq\mathcal F$.
At each iteration, the current maximiser $\hat{x}^t$ is projected onto
$\mathcal G$,
\(
   z^t:=\pi_{\mathcal G}(\hat{x}^t),
\)
computed by bisection along the ray
$\{r\hat{x}^t:r \in [0, 1]\}$ when $\{g_i\}$ are known (Appendix~\ref{app:bisection}).
By maximality of the projection, the upper cone
$(z^t,\infty)$ contains no point of $\mathcal G$ and can therefore be removed from the current polyblock without excluding feasible points.
This yields the refined vertex set $V^{t+1}$, and the process repeats until the stopping criterion is met.
The construction of $V^0, V^1, V^2$ is illustrated in Figure~\ref{fig:poa}.

\subsection{Learning monotone constraints}

\paragraph{Learned objective and constraints.}

We consider problems of the form \eqref{eq:mono_prob} in which some
objective or constraint functions may be represented by learned
surrogates rather than known analytically. Throughout this work, we
focus on the case where the objective $f$ and lower-bound constraints
$\{h_j\}$ are available to the optimiser, while the upper-bound
constraints $\{g_i\}_{i\in[m_g]}$ are learned from sampled evaluations.
The learned models may approximate the constraint functions directly,
or a derived quantity sufficient for optimisation. In either case,
their representation must preserve the monotone structure on which POA
relies. Our approach learns the radial inverse required by POA's
projection step and enforces its corresponding structural properties
architecturally.

\paragraph{Monotonic min--max networks.}
\label{ssec:mono_nets}

A useful building block for enforcing monotonicity is the min--max
network \citep{sill_monotonic},
\(
g_\theta(x)
=
\min_k\max_j
\bigl(w_{k,j}^\top x+b_{k,j}\bigr).
\)
As minima and maxima preserve coordinatewise monotonicity,
constraining every $w_{k,j}$ to be elementwise non-negative makes
$g_\theta$ non-decreasing for every parameter setting.
Hard minimum and maximum operations propagate gradient only through
active branches at almost every input, making gradient-based
training difficult. Smooth min--max architectures replace these hard
extrema by smooth approximations, retaining monotonicity
\citep{igel, expressive_mono_nets}.
We develop a smoothing appropriate for radial
inverse estimation which, in the positive bias-free case, also
preserves exact positive homogeneity.

\section{Estimating projections using learned radial inverse functions}
\label{sec:rad_inv}

While POA needs direct oracle access to the objective $f$ and lower-bound constraints $\{h_i\}_{i \in [m_h]}$, it interacts with the upper-bound constraints $\{g_i\}_{i\in [m_g]}$ only through the monotone projection $\pi_{\mathcal{G}}$ (defined in Section~\ref{ssec:mono_opt_prelim}).
A direct implementation computes $\pi_{\mathcal{G}}(x)$ by bisection along the ray $\{rx:r\ge 0\}$, repeatedly evaluating the constraints.
In our experiments, this inner bisection loop constitutes the dominant computational cost, particularly when the constraints are represented by large surrogate models.
This motivates learning the radial quantity required to evaluate the projection, rather than first learning the constraints and then performing bisection at inference time.

To enable this, we introduce the \emph{radial inverse} as an explicit learning target for approximating $\pi_{\mathcal{G}}$.
The radial inverse $\rho_g(x,y)$ is a scalar quantity that determines how far a point $x$ must be scaled along its ray to satisfy the constraint $g(\cdot)\le y$.
In contrast, the projection $\pi_{\mathcal G}(x)$ is the resulting feasible point obtained by applying this scaling.
We derive a closed-form expression for the projection in terms of the radial inverse, allowing POA to be deployed with learned projection primitives.
This replaces the iterative bisection procedure with a single forward pass of the learned model.

\subsection{Radial inverse functions}\label{ssec:rad}

For an increasing function $g:\mathbb{R}_+^n\to\overline{\mathbb{R}}$, we define its \emph{radial inverse} $\rho_g:\mathbb{R}_+^n\times\mathbb{R}\to \mathbb{R}_+\cup\{\infty\}$ as
\begin{equation}
\label{eq:rho_g}
\rho_g(x,y) := \inf\{\, r >0 : g(x/r)\le y \,\},
\end{equation}
with the convention $\inf\emptyset=\infty$, implying that $\rho_g(x,y)\in [0,\infty]$.
Let $K(y):=\{u\in\mathbb{R}_+^n: g(u)\le y\}$ denote the sub-level set of $g$ at level $y$. Then $\rho_g$ coincides with the \emph{gauge} of $K(y)$:
\begin{equation}
\label{eq:gauge_id}
\rho_g(x,y) \;=\; \gamma_{K(y)}(x),
\end{equation}
where
$\gamma_S(x) := \inf\{\, r>0 : x\in rS \,\}$ is the gauge function of a set $S \subset \mathbb{R}^n_+$ \citep{gauge_func_book}.
Intuitively, $\rho_g(x,y)$ is the minimal radial scaling $r$ for which $x/r$ satisfies the sublevel constraint $g(\cdot)\le y$. We collect several key properties of $\rho_g$ for increasing $g$:
\begin{prop}[Radial inverse properties]
\label{prop:h_properties}
Let $g:\mathbb{R}_+^n\to\overline{\mathbb{R}}$ be an increasing function, and let $x, x'\in \mathbb R^n_+$ and $y, y'\in \mathbb R$. Then $\rho_g$ satisfies:
\begin{enumerate}[label=(\roman*), leftmargin=*]
\item \textbf{(Positive-homogeneity in $x$)} For all $\alpha>0$,
$\rho_g(\alpha x, y)=\alpha\,\rho_g(x,y).$
\item \textbf{(Monotonicity)}  $\rho_g(x, y)$ is increasing in $x$ and decreasing in $y$: If $x\le x'$ then $\rho_g(x,y)\le \rho_g(x',y)$; if $y\le y'$ then $\rho_g(x,y')\le \rho_g(x,y)$.

\item \textbf{(Right-continuity in $y$)} For any $x$ and $y$,
$\lim_{\epsilon\downarrow 0}\rho_g(x,y+\epsilon)=\rho_g(x,y)$.
\end{enumerate}
\end{prop}

\subsection{From radial inverse functions to monotone projections}\label{ssec:sufficient_cond}

We now show how radial inverse evaluations are used to compute the monotone projection operator. In the context of the optimisation problem \eqref{eq:mono_prob}, the radial inverse is evaluated at the constraint thresholds $u_i$, i.e., $\rho_{g_i}(x,u_i)$. In particular, while $\rho_{g_i}(x,u_i)$ determines the required scaling for each constraint, the projection $\pi_{\mathcal G}(x)$ is obtained by combining these scalings across all constraints. This corresponds to selecting the most restrictive constraint along the ray through $x$.
\begin{prop}[Projection via radial inverses]\label{prop:proj_via_ri}
Let $\mathcal{G}$ be a compact normal set as defined in \eqref{eq:defgh}.
Then, for any $x \in \mathbb{R}^n_+ \setminus \{0\}$, the monotone projection $\pi_\mathcal{G}$ is given by
\begin{equation}\label{eq:pi_def}
\pi_{\mathcal{G}}(x)=\left(\max\nolimits_{i\in[m_g]} \rho_{g_i}(x,u_i)\right)^{-1}x,
\end{equation}
\end{prop}
For each constraint function $g_i(x)$, we propose to learn its radial inverse with a network $\phi_{\theta_i}$, i.e.,
$\rho_{g_i}(x,y)\approx \phi_{\theta_i}(x,y)$.
Using Proposition~\ref{prop:proj_via_ri},
these networks yield the approximation
\begin{equation}\label{eq:proj_approx}
    \pi_{\mathcal{G}}(x)\approx \left(\max\nolimits_{i\in[m_g]} \phi_{\theta_i}(x,u_i)\right)^{-1}x.
\end{equation}

As the threshold is an input to $\phi_{\theta_i}(x,y)$, the same learned radial inverse can be evaluated at different values of $u_i$ without retraining.

Approximating $\rho_{g_i}$ with a learned map $\phi_{\theta_i}$ naturally raises the question of when such a map can be interpreted as the radial inverse of a monotone function.
The following result shows that the properties outlined in Proposition~\ref{prop:h_properties} completely characterise radial inverse functions.
\begin{prop}[Characterisation of radial inverses]
\label{prop:inverse}
Let $h:\mathbb{R}_+^n\times\mathbb{R}\to \mathbb{R}_+\cup\{\infty\}$ satisfy properties (i)--(iii) in Proposition~\ref{prop:h_properties} for all $(x,y)\in\mathbb{R}_+^n\times\mathbb{R}$.
Then there exists an increasing function $g:\mathbb{R}_+^n\to\overline{\mathbb{R}}$ such that $h=\rho_g$. In other words, $h$ is the radial inverse of an increasing function.
\end{prop}

Proposition~\ref{prop:inverse}  motivates our model design.
If properties (i)--(iii) are satisfied exactly by each
$\phi_{\theta_i}$, then
$\phi_{\theta_i}=\rho_{\tilde g_i}$ for some increasing function
$\tilde g_i$. Consequently, the learned operator in
\eqref{eq:proj_approx} can be interpreted as the exact monotone
projection for a perturbed monotone optimisation problem in which the
original constraints $\{g_i\}$ are replaced by
$\{\tilde g_i\}$. Section~\ref{sec:hm_nets} introduces an architecture
that enforces these properties by construction.

Approximation error in the learned radial inverses therefore induces a
perturbation of the underlying normal constraint set.
Appendix~\ref{app:proof_inverse} proves Propositions ~\ref{prop:proj_via_ri} and~\ref{prop:inverse}. Appendix~\ref{app:errors_proof} bounds this perturbation in Hausdorff distance and analyses its resulting effect on the optimal objective value.

\section{Structured models for radial inverse estimation} \label{sec:hm_nets}

We introduce structured neural models for estimating the radial inverse used by POA, built from the monotonic min--max networks of Section~\ref{ssec:mono_nets}.
We first construct smooth min--max layers that preserve the structural properties required, and then use them to define the HM-RI architecture.

\subsection{Smoothed monotonic min--max layers}
\label{ssec:maxmin}

The basic building block is a min--max layer $\mu_\theta:\mathbb{R}^d \to \mathbb{R}^C$ whose $c$-th output is a minimum of maxima of affine units,
\begin{equation}\label{eq:minmax}
\mu_\theta(v)_c \;=\; \min_{k \in [K]} \, \max_{j \in [J]} \big( w_{c,k,j}^\top v + b_{c,k,j} \big),
\end{equation}
with $K$ groups of $J$ units per output feature.
As discussed in Section~\ref{ssec:mono_nets}, constraining the weights $w_{c,k,j}$ to be elementwise non-negative makes $\mu_\theta$ non-decreasing.
If the biases $b_{c,k,j}$ are omitted, each unit is linear, so \eqref{eq:minmax} is additionally positive homogeneous.
Stacks of such layers retain monotonicity, with homogeneity in the bias-free case.
To obtain informative gradients, we replace the hard $\min$ and $\max$ operations in \eqref{eq:minmax} with smooth approximations.
In each
min--max layer, the within-group maximum over the $J$ affine units is
replaced by a positive-order smooth relaxation, while the outer minimum
over the $K$ groups is replaced by the corresponding negative-order
relaxation.
We use two different relaxations depending on whether positive homogeneity is enforced.

For biased layers, the affine-unit outputs may take either sign, while positive homogeneity is not required. We therefore use a normalised log-sum-exp relaxation, following \citet{igel},
$$
\operatorname{LSE}_{p}(a)
:=
\tfrac{1}{p}\log\Big(
\tfrac{1}{J}\sum\nolimits_{j\in[J]}\exp(p\,a_j)
\Big),
\quad p\neq0,
$$
for $a=(a_1,\ldots,a_J)$.

For $p>0$, $\operatorname{LSE}_{p}$ and $\operatorname{LSE}_{-p}$
approximate the maximum and minimum, respectively, as $p\to\infty$.

For bias-free layers, we require the relaxation to preserve positive homogeneity. For $p> 0$ and $a\in\mathbb{R}_+^J$, define
\begin{align}\label{eq:powermean}
\begin{split}
M_p(a)
&:=
\left(
\tfrac1J\sum\nolimits_{j\in[J]} a_j^p
\right)^{1/p},
\\
M_{-p}(a)
&:=
\begin{cases}
\left(
\tfrac1J\sum\nolimits_{j\in[J]} a_j^{-p}
\right)^{-1/p},
& a>0,\\
0,
& \text{otherwise}.
\end{cases}
\end{split}
\end{align}

The second definition is the continuous extension of the negative-order power mean to $\mathbb{R}_+^J$. Both operators are coordinatewise increasing and positively homogeneous, with
\(
    M_p(a)\to\max_j a_j,
\)
and
\(
    M_{-p}(a)\to\min_j a_j
\)
as $p \to\infty$.
The normalised log-sum-exp and power-mean relaxations are closely related in log space; implementation details, including numerically stable evaluation and the parameterisations used by the different model variants, are given in Appendix~\ref{apen:maxmin_details}.

\subsection{Homogeneous-monotone radial inverse (HM-RI) networks}\label{ssec:hm-ri}

The HM-RI model combines two min--max stacks through a smoothed gated minimum,
\begin{equation}\label{eq:subnets}
\phi_{\theta}(x,y) \;=\; M_{-p}\big(\xi_1(x, y), \hdots, \xi_C(x, y)\big),
\end{equation}
with \; $\xi_c(x, y) = \operatorname{softplus}\big(\!-\chi_{\theta}(y)_c\big)\, \psi_\theta(x)_c,$ \;
where $\psi_{\theta}:\mathbb{R}_{+}^n \to \mathbb{R}_+^C$ is a stack of bias-free monotone min--max layers and $\chi_{\theta}:\mathbb{R}\to\mathbb{R}^C$ is a stack of biased monotone min--max layers, so that each gate is positive and non-increasing in $y$.

Every HM-RI model satisfies properties (i)--(iii) of Proposition~\ref{prop:h_properties} exactly: each $\xi_c$ is non-decreasing and positive-homogeneous in $x$, and non-increasing and continuous in $y$ -- properties inherited from $\psi_\theta$ and $\chi_\theta$ respectively.
All three properties are preserved by $M_{-p}$, which is itself monotone and homogeneous at any sharpness $p$.
By Proposition~\ref{prop:inverse}, every HM-RI model is therefore the radial inverse of an increasing function $\tilde{g}$.
The architecture is also universal within the class of finite, continuous radial inverses satisfying these structural properties, as formalised next.

\begin{prop}[Universal approximation]\label{prop:universal}
Let $\varepsilon > 0$ and suppose that $\rho : \mathbb{R}_+^n \times \mathbb{R} \to \mathbb{R}_+$ satisfies properties (i)--(iii) of Proposition~\ref{prop:h_properties} and is continuous on $\mathbb{R}_+^n \times [y_0, y_1]$. Then there is a HM-RI model \eqref{eq:subnets}, at any sufficiently large sharpness $p$, such that $|\phi_\theta(x,y) - \rho(x,y)| \le \varepsilon\, \|x\|_1$ for all $x \in \mathbb{R}_{+}^n$ and $y \in [y_0, y_1]$.
\end{prop}

The proof (Appendix~\ref{app:proof_universal}) demonstrates that neither the power means approximation nor the gated minimum combination of $x$ and $y$ reduce the expressiveness of min--max networks.

Used as surrogates for $\rho_{g_i}$ in \eqref{eq:pi_def}, the models $\phi_{\theta_i}$ thus define a true monotone projection: approximation errors perturb only the constraints, never the optimisation procedure (Section~\ref{sec:rad_inv}). Additional architectural details are provided in Appendix~\ref{app:hmri_details}.

\subsubsection{Training objective}\label{ssec:training}
The training procedure for the radial inverse surrogates $\phi_{\theta_i}$ depends on the application. For our experiments, we consider a supervised setting in which each $\phi_{\theta_i}$ is trained on a dataset $\mathcal{D}_i = \{(x_j, y_j)\}_j$ satisfying $g_i(x_j) = y_j$.

For each training pair $(x,y)$ satisfying $g_i(x)=y$, we have
$x\in K_i(y):=\{v:g_i(v)\le y\}$ and therefore
\(
    \rho_{g_i}(x,y)\le 1.
\)
Equality holds whenever $g_i$ is strictly increasing along the ray
through $x$. Thus, $1$ is an exact regression target in the strictly
increasing case, whereas for constraints with flat regions it is only
an upper bound on the true radial inverse.

For a structured HM-RI model, Proposition~\ref{prop:inverse} gives the
same condition a geometric interpretation:
$\phi_{\theta_i}(x,y)\le1$ means that $x$ belongs to the corresponding
learned sublevel set. Hence predictions above $1$ are inconsistent with
the observed fact that $x$ is feasible at level $y$.
We therefore regress $\phi_{\theta_i}(x,y)$ towards $1$ while
penalising such violations more strongly:
\begin{equation}\label{eq:loss_func}
L(\theta_i)
=
\sum_{(x,y)\in\mathcal D_i}
\left[
    \mathcal E_{\theta_i}(x,y)^2
    +
    \beta\bigl(\mathcal E_{\theta_i}(x,y)^+\bigr)^2
\right],
\end{equation}
where
$\mathcal E_{\theta_i}(x,y)=\phi_{\theta_i}(x,y)-1$
and $a^+=\max(a,0)$.
When $g_i$ is strictly increasing along the sampled rays, the targets
are exact and we use the symmetric loss ($\beta=0$).
For constraints with flat regions, such as the piecewise-constant
family of Section~\ref{sec:results}, we use $\beta>0$ to place greater
weight on respecting the known inequality
$\rho_{g_i}(x,y)\le1$.
Note that positive homogeneity fixes the model's scale along each sampled ray:
\(
    \phi_{\theta_i}(t x,y)
    =
    t\,\phi_{\theta_i}(x,y),
\)
for $t>0.$

\subsection{Structural ablations of HM-RI}
\label{ssec:simplified}

HM-RI satisfies all conditions of
Proposition~\ref{prop:inverse} exactly. To isolate the contribution of
its structural properties, we additionally evaluate three diagnostic
ablations obtained by removing positive homogeneity, monotonicity, or
both. These variants are not covered by the radial-inverse
characterisation or subsequent optimisation guarantees when a required
property is violated.
Monotonicity is removed by leaving the weight signs of
$\psi_\theta$ and $\chi_\theta$ unconstrained, while positive
homogeneity is removed by introducing bias terms in $\psi_\theta$.
The switches are independent: an unconstrained-sign, bias-free stack
remains exactly homogeneous through the rescaled smoothing described in
Appendix~\ref{apen:maxmin_details}. This gives H-RI
(homogeneity only), M-RI (monotonicity only), and RI (neither property).

All variants are trained with objective \eqref{eq:loss_func}, augmented to encourage the homogeneity which relaxed variants no longer enforce.
This augmentation proceeds by drawing a random scale $t \sim \mathcal{U}[t_{\min}, t_{\max}]$ per sample and regressing $\phi_{\theta_i}(t\,x,y)$ onto $t$.
For the exactly homogeneous models this is redundant, as it merely reweights \eqref{eq:loss_func} by $t^2$, whereas for the others it acts as data augmentation that encourages, but does not guarantee, approximately homogeneous behaviour.
Similar augmentation strategies are effective in other domains \citep{data_augmentation}.
We apply this to every variant so that the comparison remains fair, and discuss the choice of range in Appendix~\ref{apen:exp_details}.

\section{Numerical results} \label{sec:results}

We evaluate our approach on four families of monotone optimisation problems: indefinite quadratic programming, multiplicative programming, random monotone networks, and piecewise-constant (\emph{step}) constraints.
Appendix~\ref{apen:tpower} includes additional experiments for an application in cellular networks.
Across all experiments, the objective function is assumed known and the constraint set is learned from data; the same monotone min--max surrogates could also be applied to the objective or any lower-bound constraints, but we do not consider this setting here.
We set $n=4$ for all problems, with a sweep over higher dimensions reported in Appendix~\ref{sec:higher_dim}.
It is difficult to apply monotone optimisation methods to higher-dimensional problems, as the number of polyblock vertices needed to maintain the outer-approximation grows exponentially with problem size \citep{monotone_opt_tuy}.
Each instance imposes $m_g=4$ learned upper-bound constraints in addition to the box $[x_{\min}, 1]^n$, where the small offset $x_{\min}=0.1$ keeps the feasible region away from the coordinate axes, as the solver's termination guarantee requires (Theorem~\ref{thm:POA}).

To capture a diverse range of problems, we learn parameterised families of constraints $\{g_z\}_{z\in\mathbb{R}^k}$, where $z$ collects instance-specific parameters (e.g., the coefficients of a quadratic constraint, or the weights of a random monotone network). We treat \(z\) as observable instance-level side information, while \(g_z\) itself is available only through sampled evaluations.
Such families are learned by modulating each model considered by $z$, as detailed in Appendix~\ref{app:hmri_details}.
Models are trained on randomly generated samples $(x,y,z)$ satisfying $g_z(x)=y$, and each problem's constraint bounds $u_i$ are drawn such that each problem's feasible region is non-empty, but strictly smaller than $[x_{\min}, 1]^n$.
Detailed formulations of each problem class are provided in Appendix~\ref{sec:details_applications}, and the sampling procedure in Appendix~\ref{app:problem_generation}.

\paragraph{Methods.}
We evaluate the performance of different learned representations of the monotone projection operator $\pi_{\mathcal{G}}$. Specifically, we consider:
(i) learning radial-inverse surrogates using HM-RI and its variants (H-RI, M-RI, RI), which replace POA's bisection with a single forward pass per projection;
(ii) learning the constraint function directly with a monotone min--max surrogate (M-Net) and computing projections by bisection on the surrogate;
and (iii) an oracle projection computed by bisection on the ground-truth constraint functions (\textsc{Oracle}). All learned methods share the same $z$ and shared buffer of sampled constraint data; only \textsc{Oracle} accesses the ground-truth constraints.

As additional baselines, we include Sequential Least Squares Programming (SLSQP) \citep{slsqp} and Constrained Optimization BY Linear Approximation (COBYLA) \citep{Cobyla} from the \texttt{scipy.optimize} package \citep{scipy}. These are local constrained solvers applied in a learn-then-optimise pipeline: the known objective is optimised subject to the same learned M-Net surrogate constraint. While they share the surrogate, they perform only local search and do not exploit monotonic structure for global optimisation.

\paragraph{Evaluation metric.}
We report the \emph{projected objective} value for the returned solution $x^\star$, computed by bisection on the ground-truth constraints along the ray anchored at the box's lower corner,
\begin{align*}
    \mathrm{P\text{-}Obj}(x^\star) = \max \{ f(x_t) \,:\, t \ge 0,\ x_t \in \mathcal{G} \cap [x_{\min}, 1]^n \},
\end{align*}
where $x_t := x_{\min} + t(x^\star - x_{\min})$.
Thus, $\mathrm{P\text{-}Obj}(x^\star)$
is the best feasible objective value along the ray joining $x_{\min}$ and $x^\star$, providing a unified measure of solution quality across methods (Figure~\ref{fig:results_summary}).
We report it as a percentage of the \textsc{Oracle} reference objective, whose absolute value is reported alongside.

Because the point is re-projected onto the true constraint boundary, this metric scores only the \emph{direction} of $x^\star$ from $x_{\min}$ and is insensitive to how far along that ray a method stops.
We therefore also report, in Figure~\ref{fig:tradeoff}, a scatter plot of the unprojected objective $f(x^\star)$ against total constraint violation $\sum_i [g_i(x^\star) - u_i]^+$, which together highlight the objective/feasibility trade-off that the projected objective hides.

\paragraph{Protocol and data regimes.}
Models are trained on fixed sets of either 512 (\emph{Limited}) or 8192 (\emph{Large}) samples, shared across learned methods.
For each family we solve the same 50 random instances and repeat training over three seeds, reporting the mean over instances $\pm$ the standard error over seeds.
Every method gives a solution on every instance, which for the direct surrogate requires the fallback of Appendix~\ref{apen:poa_implement}.
Details, and further analysis are given in Appendices~\ref{apen:exp_details} and~\ref{apen:results}.

\begin{figure*}[t]
    \centering
    \includegraphics[width=\linewidth]{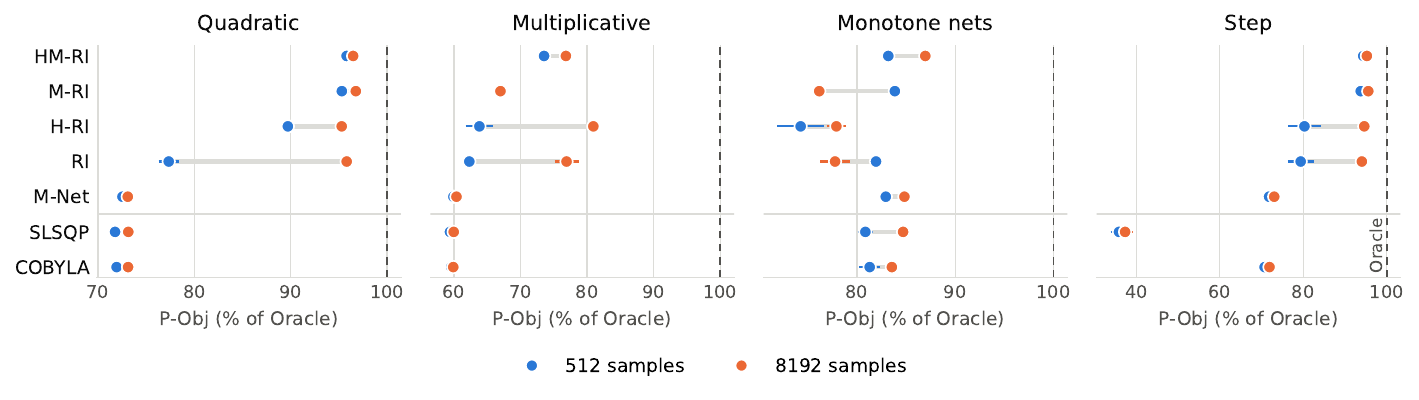}
    \caption{
    Objective as a percentage of the \textsc{Oracle}'s, for each method and problem at $n=4$.
    Each method is shown at both data regimes, connected by a bar, so that both the level attained and its sensitivity to training-data size are visible.
    Horizontal ticks are errors over three training seeds.
    }
    \label{fig:results_summary}
\end{figure*}

\begin{figure*}[t]
    \centering
    \includegraphics[width=\linewidth]{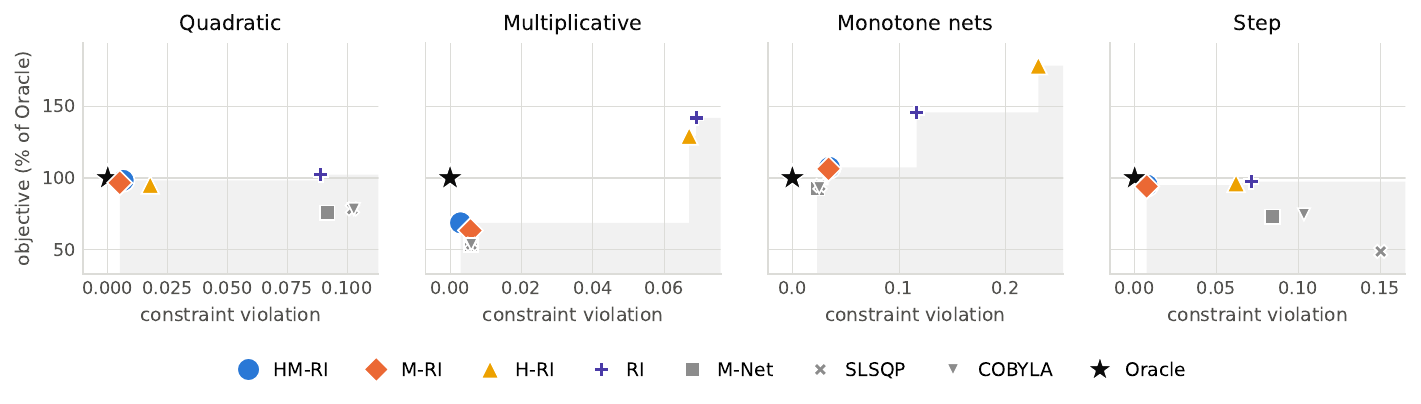}
    \caption{
    Objective at the returned point (as a percentage of the \textsc{Oracle}'s) against its total constraint violation, per method and problem family at $n=4$ in the limited-data regime.
    The four radial-inverse variants are shown in colour while the remaining methods (M-Net and the two local solvers) share a single grey and are distinguished by marker shape.
    The grey regions are those dominated by some method at both metrics.
    Runtimes are in Table~\ref{tab:runtimes}.
   }
    \label{fig:tradeoff}
\end{figure*}

\paragraph{Accuracy.}
Figure~\ref{fig:results_summary} shows a consistent ordering.
Averaged over the eight settings, HM-RI recovers $87.8\%$ of the oracle objective, ahead of its single-property ablations M-RI ($84.4\%$), H-RI ($82.1\%$) and RI ($80.7\%$), the direct surrogate M-Net ($72.3\%$), and the local baselines COBYLA ($71.5\%$) and SLSQP ($62.9\%$).
Differences among the radial-inverse variants are often within instance-level noise, so Appendix~\ref{apen:significance} reports paired confidence intervals over the 50 shared instances.
HM-RI is ahead of the direct surrogate and both local baselines in every comparison, significantly so in more than half of them.

Multiplicative problems in the limited-data regime separate the ablations most cleanly.
Their radial inverse has a closed form, $\rho(x,y) = (g(x)/y)^{1/d}$ with the homogeneity degree $d$ drawn per instance.
The performance of the four RI variants is ordered strictly by how much structure they retain, with HM-RI ($73.5\%$) well ahead of the others.
They are also the most difficult class of problems considered, with the largest average gap in optimality between \textsc{Oracle} and the remaining baselines.

Step function problems highlight a separate advantage of POA: its convergence guarantee does not require smoothness or continuity (Theorem~\ref{thm:POA}).
With a piecewise-constant objective, finite differences carry no information, so SLSQP terminates almost where it started, recovering only $35.8\%$ and $37.3\%$ of the oracle, while HM-RI leads SLSQP by $58.6$ and $57.9$ points.
COBYLA, being derivative-free, is less affected by the flatness and retains its performance ($70.7\%$).

Figure~\ref{fig:tradeoff} illustrates how structural properties pay off, and why the projected objective alone understates them.
The shaded region is dominated: every point inside it is beaten simultaneously on objective and on constraint violation by some other method.
Both local baselines and the direct surrogate fall inside it on the quadratic, multiplicative and step classes.
More structured RI models rarely dominate their unstructured counterparts, but observe the trade-off between objective and constraint violation: as the model drops monotonicity, the objective improves at the cost of larger constraint violations.
This behaviour is consistent with the stability analysis of
Appendix~\ref{app:errors_proof}, which establishes controlled
feasible-set perturbations,
and corresponding objective perturbations under conditions.
In the experiments, which additionally impose a lower box and include
piecewise-constant objectives, the structured radial-inverse models
likewise exhibit favourable objective--violation trade-offs.

\paragraph{Generalisation under limited data.}
Structure helps most in the limited-data regime, where HM-RI is the best method on 3/4 families, and its projected objective is stable or improving between the two regimes on every one.
Ablations catch up in the large-data regime, and on the multiplicative family H-RI overtakes HM-RI once there is enough data to fit the instance-dependent power law without the structural prior.
This is the expected effect of inductive bias, which is most beneficial when data is scarce and does not become a liability overall.
HM-RI remains the most consistent method across all families and regimes, never performing far from the best baseline.

Notably, held-out loss and solver quality are almost uncorrelated.
HM-RI's scale-regression loss ranges from best to worst across settings (Table~\ref{tab:loss_errors}), yet its projected objective is consistently high.
Held-out regression error is a poor proxy for performance inside POA, which may query the RI model on a different distribution of points than the training/test set.
This is particularly harmful for non-homogeneous models, which do not naturally extrapolate to the infeasible points $\rho(x,y) > 1$ which POA often queries.
This explains why the performance of such models stagnates with more data on monotone network problems, despite a reduction in held-out loss.
See Appendix~\ref{app:training_data_analysis} for further discussion and Appendix~\ref{apen:ablations} for ablations measuring how strongly unconstrained variants violate each dropped property.

\paragraph{Runtime.}
RI-based methods eliminate the inner bisection loop of each POA iteration, replacing multiple surrogate evaluations per projection with a single batched forward pass.
HM-RI and M-Net run an identical solver and differ only in this primitive: HM-RI takes $0.15$--$0.41$\,s per solve against M-Net's $0.68$--$1.30$\,s, a speed-up of $3.0$--$4.7\times$ (Table~\ref{tab:runtimes} in Appendix~\ref{apen:runtimes}), obtained simultaneously with the accuracy gains above.
Note that \textsc{Oracle} remains faster than either approach ($0.10$--$0.19$\,s), since the classes of constraint functions considered are considerably cheaper to evaluate than any learned surrogate.
Local optimisation methods are faster still in wall-clock terms ($0.02$--$0.36$\,s) but operate under a different trade-off: they optimise surrogate models locally, without global optimality guarantees, and they are the methods that lose most on the step family.
Implementation considerations are given in Appendix~\ref{apen:poa_implement}.

\section{Conclusion}
We presented a framework for integrating learned components into the POA algorithm for monotone optimisation, showing that learning the radial inverse, the primitive POA uses to project onto normal sets, is more effective than approximating constraint functions directly.
To enable this, we introduced the HM-RI architecture, together with necessary and sufficient conditions under which a learned map is the radial inverse of an increasing constraint.
The architecture satisfies these conditions by construction ensuring every trained model is compatible with POA.

Across four problem families, HM-RI consistently outperforms local constrained solvers and improves both runtime and projected objective over constraint-function surrogates.
The ablations locate the benefit in the structural properties rather than the radial-inverse formulation alone, and
show that conventional model-quality metrics can be misleading: variants with lower regression error can nevertheless perform worse inside the solver.
Enforcing both properties is also what makes performance robust: each single-property variant is strong on some families and weak on others, whereas the full model is near the best on all of them.

\bibliographystyle{plainnat}
\bibliography{ref}

\newpage

\appendix

\section{Problem formulation details}\label{sec:details_applications}
In this section, we provide additional details on the problem families considered in our experiments. Each family defines a parameterised class of increasing functions $\{g_z\}$ used to generate objectives and constraints: for every instance, the objective and the upper-bound constraint are drawn independently from the same family, the objective is treated as known, and the constraint is learned from samples. In all cases, parameters are sampled so that $g_z$ is increasing on the box, and constraint thresholds are drawn by the corner-scaling rule of Section~\ref{sec:results}, which pins how tightly each constraint binds.

\subsection{Indefinite quadratic programming} \label{ssec:quad_prog}

We consider problems of the form
\begin{align}\label{eq:qp}
\begin{split}
    \max_{x\in [x_{\min},1]^n}& \quad x^T Q_0 x + c_0^T x, \\
    \text{s.t.}& \quad x^T Q_i x + c_i^Tx \leq u_i, \ i \in [m_g]
\end{split}
\end{align}
where $Q_0,Q_i \in\mathbb{R}_+^{n\times n}$ and $c_0,c_i\in\mathbb{R}_+^n$ for $i\in[m_g]$.
In our generator, each coordinate is given an anisotropic scale $s_k = \sigma^{\,\omega_k}$ with $\omega_k \sim \mathcal{U}[-\tfrac12, \tfrac12]$ and spread $\sigma = 100$; entries are then sampled as $c_k \sim s_k\,\mathcal{U}[0,1]$ and $Q_{kl} \sim s_k s_l\,\mathcal{U}[0,1]$, and the parameters are rescaled so that $g_z(x_{\max}) = 1$. The instance parameters are $z = (Q, c)$.
The anisotropy is essential, rather than cosmetic. With i.i.d.\ coefficients ($\sigma = 1$) the family is nearly permutation-symmetric, and concentration makes $x^\top Q x + c^\top x$ almost a function of $\sum_k x_k$ alone: we measured a rank-$R^2$ of $0.996$ against $\sum_k x_k$ at $n=4$, so all rays through the box are close to equivalent, a trivial diagonal ray already attains $87\%$ of the oracle objective, and every method is compressed into a narrow band regardless of its quality. At $\sigma = 100$ this rank-$R^2$ falls to $0.71$ ($n=4$) and $0.34$ ($n=5$), giving level sets whose shape genuinely varies with direction.
The non-negativity of the coefficients ensures monotonicity; however, general polynomial problems can be transformed into monotone form
\citep{monotonic_opt_tuy2}. Non-convex quadratic programs form a widely studied subclass of such problems \citep{nonconvex_quad_heuristics}. For the constraints in \eqref{eq:qp}, the radial inverse is tractable since it reduces to solving a quadratic equation, making this family a useful sanity check.

\subsection{Multiplicative programming} \label{ssec:multiplicative}
Products of powers are a classical monotone-optimisation class \citep{monotone_opt_tuy}, and are structurally unlike the other families here. We take
\begin{align*}
    g_z(x) = \prod\nolimits_{i} \big(x_i / x_{\max,i}\big)^{b_i}, \qquad b_i \sim \mathcal{U}[0.25, 1.75],
\end{align*}
with $z = b$. Writing the product as a ratio to $x_{\max}$ normalises it for free, $g_z(x_{\max}) = 1$, without carrying a scale parameter.

This family is exactly homogeneous of degree $d = \sum_i b_i$, drawn independently per instance (measured range $1.8$--$6.3$ at $n=4$), and it is the only family whose radial inverse is available in closed form,
\begin{align*}
    \rho_{g_z}(x, y) = \big(g_z(x) / y\big)^{1/d}.
\end{align*}
That makes it diagnostic rather than merely hard: a surrogate must realise a power law in $y$ whose exponent varies from instance to instance, so the difficulty is localised in the $y$/$z$ conditioning path while $\psi_\theta$'s exact degree-one homogeneity in $x$ remains correct. It is also the most direction-sensitive family in the suite -- a uniformly random point scores only $21.7\%$ of the oracle objective, against $42.6\%$ for the quadratic and $65.6\%$ for the monotone networks -- so the projected objective has the widest range in which to separate methods. The cost is a wide dynamic range: constraint values span roughly $4.6$ decades inside the box, which is part of why it is the hardest family for every method.

\subsection{Random monotone networks} \label{ssec:rand_nets}
As a generic family of monotone constraints without exploitable closed-form structure, we use randomly initialised two-layer ReLU networks with non-negative weights,
\begin{align*}
    N_z(x) = a^\top \big( W x / n + b \big)^+,
\end{align*}
where $W \in [0,1]^{h \times n}$ and $a \in [0,1]^{h}$ are sampled uniformly, $h = 16$, and the instance parameters are $z = (W, b, a)$. The ReLU thresholds are non-positive, $b \in [-\tfrac12, 0]^h$, so that $N_z(0) = 0$: a positive bias would act as an $x$-independent offset, making levels below it unreachable along every ray. Finally, the output is normalised by its value at the top corner of the box, $g_z(x) = N_z(x) / N_z(x_{\max})$, so that every instance sweeps the range $[0,1]$ across the box.

\subsection{Piecewise-constant (step) constraints} \label{ssec:steps}
To model constraints arising from systems with discrete configurations (e.g., quantised operating modes or rate selection), we quantise the indefinite quadratic family of Section~\ref{ssec:quad_prog} to steps of width $\Delta = 0.02$ in output units:
\begin{align*}
    g_z(x) = \Delta \left\lceil \big(x^\top Q x + c^\top x\big)\,\Delta^{-1} \right\rfloor,
\end{align*}
where $\lceil \cdot \rfloor$ denotes rounding to the nearest integer. Both the constraints \emph{and} the objective are quantised in this way: POA requires only the ordering of objective values, so a piecewise-constant objective poses it no difficulty, whereas it removes the gradient information local solvers rely on. This makes the family a sharp test of the value of exploiting monotone structure. It is also doubly challenging for the learned models: the flat regions mean the constraints are not strictly monotone, so radial-inverse training targets are upper bounds rather than exact values, exercising the asymmetric loss ($\beta > 0$) of Section~\ref{ssec:training}.

The quadratic rather than the monotone-network family is the one we quantise, and the choice matters for what the experiment can show. Quantisation makes the objective landscape hard to search, but it leaves the level-set \emph{geometry} of the underlying family untouched, and the projected objective scores only the direction of the returned point. Both quantised families punish gradient-based search severely, but only the quadratic one leaves the metric enough range to display it: the monotone-network geometry is so insensitive to direction that a random point already attains about $65\%$ of the oracle objective, against about $43\%$ on the anisotropic quadratic geometry, so quantising the monotone networks would compress every method into a narrow band regardless of quality.

The step width trades plateau size against metric resolution -- objective values can only land on multiples of $\Delta$, so a coarse $\Delta$ makes the projected objective tie across instances -- and $\Delta = 0.02$ keeps $50$ distinct levels across the box.

\section{Proofs}

This section provides proofs for Propositions~\ref{prop:proj_via_ri}, \ref{prop:inverse}, \ref{prop:universal}, and \ref{prop:error_propagation}.

\subsection{Proof of radial inverse properties}\label{app:proof_prop}

\begin{proof}[Proof of Proposition \ref{prop:h_properties}]
We demonstrate the three properties in order. Throughout, note that the sub-level sets $K(y)=\{u\in\mathbb{R}^n_+: g(u)\le y\}$ and the comparisons $g(u)\le y$ are well defined for extended-valued $g$ and $y\in\mathbb{R}$, so no step below requires $g$ to be finite.
\begin{itemize}[leftmargin=0pt]
    \item[] (\textbf{Positive-homogeneous in $x$}) - For any $\alpha > 0$, $x\in \mathbb R^n_+$ and $y\in \mathbb R$, we have by the definition of $\rho_g$ in \eqref{eq:rho_g}:
    \begin{align*}
    \rho_g(\alpha\,x,y) &= \inf \{r>0 : g(\alpha x / r) \leq y \} \\
    &= \inf \{r>0 : g( x / (r/\alpha) ) \leq y \} \\
    &= \alpha\, \rho_g(x,y).
    \end{align*}

    \item[] (\textbf{Monotonicity}) - We show monotonicity in $x$ first. Let $x,x' \in \mathbb R^n_+$ with $x' \geq x$. Since $g$ is increasing, we have $g(x' / r) \geq g(x/r)$ for any $r>0$. Then,
    \begin{align*}
        \{r>0 : g(x' / r) \leq y \} \subset  \{r>0 : g(x / r) \leq y \}
    \end{align*}
    for any $y \in \mathbb{R}$. It follows that
    \begin{align*}
    \rho_g(x',y) &= \inf \{r>0 : g(x' / r) \leq y \} \\&\geq  \inf \{r>0 : g(x / r) \leq y \}= \rho_g(x,y).
    \end{align*}

    The monotonicity in $y$ follows in a similar way: For $y, y'\in \mathbb R$ with $y \leq y'$ and all $x \in\mathbb R^n_+$, we have
    \begin{align*}
        \{r>0 : g(x / r) \leq y \} \subset  \{r>0 : g(x / r) \leq y' \},
    \end{align*}
    therefore
    \begin{align*}
       \rho_g(x,y) &= \inf \{r>0 : g(x / r) \leq y \} \\&\geq \inf \{r>0 : g(x / r) \leq y' \} = \rho_g(x,y').
    \end{align*}

    \item[] (\textbf{Right-continuity in $y$}) - Fix $x \in \mathbb{R}^n_+$, $y \in \mathbb{R}$, and let $r = \rho_g(x,y)$.
    We consider separately the cases $r=0$, $0<r<\infty$, and
    $r=\infty$.

    If $r=0$, monotonicity of $\rho_g$ in $y$ gives, for every
    $\epsilon>0$,
    \[
        0
        \leq
        \rho_g(x,y+\epsilon)
        \leq
        \rho_g(x,y)
        =
        0.
    \]
    Hence $\rho_g(x,y+\epsilon)=0$, and right-continuity follows
    immediately.

    Now suppose $0<r<\infty$. For any $\delta\in(0,r)$, the definition of
    the gauge implies
    $x\notin (r - \delta)K(y).$ Further, the sublevel sets satisfy
    \begin{align}\label{eq:k_cap}
    K(y) = \bigcap_{\epsilon > 0} K(y + \epsilon).
    \end{align}
    Hence,
    there exists some $\epsilon > 0$ depending on $\delta$ such that for all $\epsilon' \in [0, \epsilon)$, we have $$ x\notin (r - \delta)K(y + \epsilon').$$ Using \eqref{eq:gauge_id}, we then compute
    \begin{align*}
        \rho_g(x, y + \epsilon') &= \gamma_{K(y+\epsilon')}(x) \\
        &=\inf \{ r>0 : x\in r K(y+\epsilon')\}\\
        &\geq  r - \delta.
    \end{align*}
    As $r= \rho_g(x,y)$, the monotonicity of $\rho_g$  in $y$ implies $\rho_g(x, y + \epsilon')\leq r$, and hence $$\rho_g(x, y + \epsilon') \in [r - \delta, r].$$ Since $\delta$ is arbitrary, this proves right-continuity.

    Finally, suppose $r=\infty$. For any $M>0$ we have $x\notin M K(y)$, and by \eqref{eq:k_cap} $x\notin M K(y+\epsilon)$ holds for all sufficiently small $\epsilon>0$, implying $\rho_g(x,y+\epsilon)\ge M$; since $M$ is arbitrary, $\rho_g(x,y+\epsilon)\to\infty$ as $\epsilon\to 0$.
    \end{itemize}

\end{proof}

\subsection{Proof of projections as radial inverse functions}\label{app:proof_inverse}
In this section, we provide the proofs of Proposition \ref{prop:proj_via_ri} and Proposition \ref{prop:inverse}.

\begin{proof}[Proof of Proposition \ref{prop:proj_via_ri}]

Recall the definition (Section \ref{ssec:mono_opt_prelim}) of monotone projection onto a compact normal set $\mathcal{G}$:
$$\pi_{\mathcal{G}}(x)=\alpha_\mathcal{G}(x)\, x,\ \ \alpha_\mathcal{G}(x) =\sup\{\alpha\ge 0:\ \alpha x\in\mathcal{G}\},$$
and the gauge function:
$$\gamma_\mathcal{G}(x) = \inf\{r>0:\ x\in r\, \mathcal{G}\}.$$
By the compactness of $\mathcal{G}$, it follows that $\gamma_\mathcal{G}(x) > 0$ for every $x \in \mathbb{R}^n_+ \setminus \{0\}$, since $\gamma_\mathcal{G}(x) = 0$ would force the unbounded ray $\{rx: r>0\}$ into $\mathcal{G}$.
With the conventions $0^{-1}=\infty$ and $\infty^{-1}=0$, we have $\alpha_{\mathcal G}(x)= (\gamma_{\mathcal G}(x))^{-1}$ for $x\in\mathbb{R}_+^n\setminus\{0\}$.

For the normal set $\mathcal{G}$, we compute
\begin{align*}
    \alpha_\mathcal{G}(x) &= \sup\{\alpha \ge 0: g_i(\alpha x) \leq u_i, \, i \in [m_g] \} \\
    &= \sup\{\alpha \ge 0: \max\nolimits_{i \in [m_g]} (g_i(\alpha x) - u_i) \leq 0 \} \\
    &= \min_{i \in [m_g]} \left( \sup\{\alpha \ge 0: g_i(\alpha x) - u_i \leq 0 \} \right) \\
    &= \min_{i \in [m_g]} \alpha_{\mathcal{G}_i}(x),
\end{align*}
where
\begin{align*}
        \mathcal{G}_i = \{x \in \mathbb{R}_+^n : g_i(x) \le u_i\}.
\end{align*}
The radial inverse satisfies $\rho_{g_i}(x,u_i)=\gamma_{\mathcal G_i}(x)$. This yields:
\begin{align*}
    \pi_\mathcal{G}(x) &=\alpha_{\mathcal G}(x)x\\
    &= \left( \min\nolimits_{i \in [m_g]}\alpha_{\mathcal{G}_i}(x) \right)\,x \\
    &= \left( \min\nolimits_{i \in [m_g]}[\gamma_{\mathcal{G}_i}(x)]^{-1} \right)\,x \\
    &= \left( \min\nolimits_{i \in [m_g]}[\rho_{g_i}(x,u_i)]^{-1} \right)\,x \\
    &= \Bigl(\max\nolimits_{i\in[m_g]} \rho_{g_i}(x,u_i)\Bigr)^{-1}\,x.
\end{align*}
\end{proof}

\begin{proof}[Proof of Proposition \ref{prop:inverse}]
First, we show that positive-homogeneity implies that $h$ is a gauge function:
\begin{align*}
    h(x, y) &= \inf \{r > 0 : h(x,y) \leq r \} \\
    &= \inf \{r > 0 : h\left(\frac{x}{r},y\right) \leq 1 \} \\
    & = \inf \{r > 0 : x \in rK(y) \}  \\
    &= \gamma_{K(y)}(x)
\end{align*}
where $$K(y) = \{z\in \mathbb{R}^n_+: h(z,y) \leq 1 \}.$$ By $y$-monotonicity of $h$, these sets are nested -- meaning that for $y' \geq y$ we have $$K(y) \subset K(y').$$

We then show that $K(y)$ are the sub-level sets of a function $g$, which we define as:
$$g(x) :=  \inf \{y \in \mathbb{R}: x \in K(y)\} \;\in\; \overline{\mathbb{R}},$$

with the conventions $\inf\emptyset = +\infty$ and $\inf\mathbb{R} = -\infty$.
Recall that $g$ is a valid constraint function since \eqref{eq:mono_prob} admits such extended-valued constraints.
Right-continuity of $h$ then identifies $K(y)$ with the sub-level sets of $g$:
\begin{alignat*}{2}
    g(x)\le y &\iff x\in K(y+\eta) &&\quad \forall\eta>0 \\
    &\iff h(x,y+\eta)\le1 &&\quad \forall\eta>0 \\
    &\iff h(x,y)\le1 \\
    &\iff x\in K(y),
\end{alignat*}
where the first equivalence follows from the nestedness of $K(y)$, and the third from right-continuity in the forward direction and $y$-monotonicity in the reverse direction.

It remains to show that $g$ is increasing.
Fix $x\le x'$, so that for every $y\in\mathbb{R}$,
\begin{align*}
    x\le x'
    &\implies h(x,y)\le h(x',y) \\
    &\implies \bigl(x'\in K(y)\ \Rightarrow\ x\in K(y)\bigr),
\end{align*}
by $x$-monotonicity and the definition of $K(y)$ respectively.
Hence $\{y:x'\in K(y)\}\subseteq\{y:x\in K(y)\}$, and taking infima gives $g(x)\le g(x')$.
Finally, combining the gauge identity $h(\cdot,y) = \gamma_{K(y)}$ with the fact that $K(y)$ is the sub-level set of $g$ at level $y$ gives $h(x,y) = \gamma_{K(y)}(x) = \rho_g(x,y)$ for all $(x,y)$, as required.
\end{proof}

\subsection{Proof of universal approximation}\label{app:proof_universal}

\begin{proof}[Proof of Proposition \ref{prop:universal}]

Since $\rho$ and every HM-RI model are positively homogeneous in $x$,
it suffices to prove the approximation on the closed simplex
\[
    \Delta
    :=
    \{x\in\mathbb{R}_+^n:\mathbf{1}^\top x=1\}.
\]
Indeed, for any $x\neq0$, writing
$\bar{x}:=x/\|x\|_1\in\Delta$, positive homogeneity gives
\[
|\phi_\theta(x,y)-\rho(x,y)|
=
\|x\|_1
|\phi_\theta(\bar{x},y)-\rho(\bar{x},y)|.
\]
Thus, a uniform bound
$|\phi_\theta-\rho|\le\varepsilon$ on
$\Delta\times[y_0,y_1]$ implies the stated bound
$|\phi_\theta(x,y)-\rho(x,y)|
\le\varepsilon\|x\|_1$ for all $x\neq0$.
At $x=0$, both functions vanish, so the bound holds trivially.

\emph{An approximate gauge model.}
Fix $\delta \in (0,1]$ and let $V = \{\delta, 2\delta, \dots\}^n \cap (0, 1+\delta]^n$ be a grid of strictly positive points, with one component $\xi_v$ of \eqref{eq:subnets} per grid point.
We begin by assuming all min--max operations are exact; smoothing will be handled in the next step.
Each map $y \mapsto \rho(v,y)$ is non-negative, continuous, and non-increasing: a class of functions uniformly approximated by softplus composed with an exact monotone min--max network \citep{sill_monotonic, min-max_nets}.
Let $\tau = \varepsilon\delta/4$, and choose $\chi_\theta$ so that the gates $a_v := \operatorname{softplus}(-\chi_\theta(\cdot)_v)$ satisfy $|a_v - \rho(v,\cdot)| \le \tau$ on $[y_0,y_1]$.
For each $v \in V$, take the feature $\psi_\theta(x)_v = \max_i x_i / v_i$ as a bias-free max-group, giving the model
\begin{equation*}
\mu(x,y) \;=\; \min_{v \in V}\; a_v(y)\, \max_{i \in [n]}\, \frac{x_i}{v_i}.
\end{equation*}

We bound $\mu$ away from $\rho$ on $\Delta \times [y_0,y_1]$, noting that $\max_i x_i/v_i \le 1/\delta$ there, so each gate error perturbs a component by at most $\tau/\delta = \varepsilon/4$. For the lower bound, fix $(x,y)$ and $v \in V$, and set $t = \max_i x_i / v_i$. Then $t\,v \ge x$, so monotonicity and homogeneity of $\rho$ in $x$ give
\begin{equation*}
\rho(x,y) \;\le\; \rho(t\,v, y) \;=\; t\, \rho(v,y) \;\le\; a_v(y) \max_{i \in [n]} \frac{x_i}{v_i} + \frac{\tau}{\delta},
\end{equation*}
and minimising over $v$ yields $\rho \le \mu + \varepsilon/4$. For the upper bound, let $v \in V$ round each coordinate of $x$ up to the grid, so that $v \ge x$ and $\|v - x\|_\infty \le \delta$. Then $\max_i x_i/v_i \le 1$, hence
\begin{equation*}
\mu(x,y) \;\le\; a_v(y) \;\le\; \rho(v,y) + \tau \;\le\; \rho(x,y) + \omega(\delta) + \varepsilon/4,
\end{equation*}
where $\omega$ is the modulus of continuity of $\rho$ on the compact set $[0,2]^n \times [y_0,y_1]$. Choosing $\delta$ with $\omega(\delta) \le \varepsilon/4$ gives $|\mu - \rho| \le \varepsilon/2$ on $\Delta \times [y_0, y_1]$.

\emph{Smoothing.}
It remains to replace the hard min--max operations in the construction
above by the smooth operators used by HM-RI.

For the homogeneous stack $\psi_\theta$, maxima and minima are replaced
by positive- and negative-order power means. For non-negative arguments,
\[
J^{-1/p}\max_j a_j
\;\le\;
M_p(a)
\;\le\;
\max_j a_j,
\]
\[
\min_k a_k
\;\le\;
M_{-p}(a)
\;\le\;
K^{1/p}\min_k a_k,
\]
where the second inequality also holds under the continuous extension
of $M_{-p}$ at zero. Hence these approximations converge uniformly to
the corresponding hard extrema as $p\to\infty$ on bounded sets.

For the biased stack $\chi_\theta$, maxima and minima are instead
replaced by the normalised log-sum-exp relaxation. For $p>0$,
\[
\max_j a_j-\frac{\log J}{p}
\;\le\;
\operatorname{LSE}_{p}(a)
\;\le\;
\max_j a_j,
\]
and
\[
\min_k a_k
\;\le\;
\operatorname{LSE}_{-p}(a)
\;\le\;
\min_k a_k+\frac{\log K}{p}.
\]
Thus the smoothed $\chi_\theta$ converges uniformly to its hard
min--max counterpart as $p\to\infty$. Since softplus is
$1$-Lipschitz, the corresponding gates converge uniformly as well.

Because the construction uses finitely many units and all inputs lie
in the compact set $\Delta\times[y_0,y_1]$, we may therefore choose
$p$ sufficiently large that replacing the hard operations inside
$\psi_\theta$ and $\chi_\theta$ changes the resulting model by at most
$\varepsilon/4$ uniformly on this set. Denote the resulting model,
with the outer minimum still exact, by $\mu_p$. Then
\[
    |\mu_p(x,y)-\rho(x,y)|
    \le \frac{3\varepsilon}{4}
\]
on $\Delta\times[y_0,y_1]$.

Let
\[
    B_0
    :=
    \sup_{\Delta\times[y_0,y_1]}\rho
    +\frac{3\varepsilon}{4}
    <\infty.
\]
Then $\mu_p(x,y)\le B_0$ uniformly for all sufficiently large $p$.

Finally, replace the outer minimum over the $|V|$ positive components
by $M_{-p}$.
The power-mean bound gives
\[
0
\le
\phi_\theta(x,y)-\mu_p(x,y)
\le
\bigl(|V|^{1/p}-1\bigr)B_0.
\]
Increasing $p$ further if necessary so that
\[
    \bigl(|V|^{1/p}-1\bigr)B_0\le\frac{\varepsilon}{4}
\]
yields
\[
    |\phi_\theta(x,y)-\rho(x,y)|
    \le\varepsilon
\]
uniformly on $\Delta\times[y_0,y_1]$. The claimed global bound then
follows from the homogeneity argument above.

\end{proof}

\subsection{Error analysis and additional results}
\label{app:errors_proof}

This section collects additional theoretical results supporting the use of learned radial inverse functions for projection-based optimisation. In particular, we analyse how approximation errors in the learned radial inverse affect (i) the projection operator, (ii) the geometry of the feasible set, and (iii) the resulting optimisation objective. Together, these results show that projection-based learning yields stable behaviour, with controlled deviations in both feasible sets and objective values. We only consider the case where no lower-bound constraints are present (i.e. where $\mathcal{H} = \mathbb{R}^n_+$), as such constraints can be a source of instability in POA as demonstrated in \citet{sit_algorithm}.

In general, small errors in constraint function approximation can lead to large deviations in the feasible set. In contrast, projection-based learning induces controlled perturbations in both feasible sets and objective values, as approximation errors translate directly into controlled geometric errors. To demonstrate this, we measure set approximation using the Hausdorff distance
\[
d_H(A,B) := \max\left\{\sup\nolimits_{a\in A} d(a,B),\; \sup\nolimits_{b\in B} d(b,A)\right\},
\]
where $A, B \subset \mathbb{R}^n$, and $d(x,A) = \inf_{a\in A} \|x-a\|_1$ is induced by the $\ell_1$ norm.
\begin{prop}[Stability of the feasible set and the optimum]
\label{prop:stability}
Let $\mathcal{G}, \tilde{\mathcal{G}} \subset \mathbb{R}^n_+$ be compact normal sets whose monotone projections are uniformly close,
\[
\sup\nolimits_{x \in \mathbb{R}^n_+} \|\pi_{\mathcal{G}}(x) - \pi_{\tilde{\mathcal{G}}}(x)\|_1 < \epsilon .
\]
Then $d_H(\mathcal{G}, \tilde{\mathcal{G}}) < \epsilon$.
If moreover $f$ is $\lambda$-Lipschitz, and
\[
x^* \in \arg\max_{x \in \mathcal G} f(x),
\quad
\tilde{x}^* \in \arg\max_{x \in \tilde{\mathcal G}} f(x),
\]
then $|f(x^*) - f(\tilde{x}^*)| < \lambda \epsilon$.
\end{prop}
\begin{proof}
Recall the definition (Section~\ref{ssec:mono_opt_prelim}) of the monotone projection onto a normal set $T$,
\[\pi_{T}(x)=\alpha_T(x)\, x,\qquad \alpha_T(x) =\sup\{\alpha\ge 0:\ \alpha x\in T\},\]
and write $S := \sup_{x \in \mathbb{R}^n_+}\|\pi_{\mathcal{G}}(x) - \pi_{\tilde{\mathcal{G}}}(x)\|_1 < \epsilon$ for the uniform projection error.

\emph{(Feasible sets.)} Let $x \in \tilde{\mathcal G}$.
If $x \in \mathcal G$ then $d(x,\mathcal G)=0$, so assume $x \notin \mathcal G$; then $\alpha_{\tilde{\mathcal G}}(x) \ge 1 \ge \alpha_{\mathcal G}(x)$, and since $\pi_{\mathcal G}(x) \in \mathcal G$,
    \begin{align*}
        d(x, \mathcal{G}) &\leq \|x - \pi_\mathcal{G}(x) \|_1 \\
         &= |1 - \alpha_\mathcal{G}(x)|\, \|x\|_1\\
         &\leq | \alpha_{\tilde{\mathcal{G}}}(x) - \alpha_\mathcal{G}(x)| \, \|x\|_1\\
         &= \| \pi_{\tilde{\mathcal{G}}}(x) - \pi_\mathcal{G}(x)\|_1 \\
         &\leq S.
    \end{align*}
    A symmetric argument also shows that $d(y, \tilde{\mathcal{G}}) \leq S$ for all $y \in \mathcal{G}$.
    Since these bounds are uniform in $x$ and $y$, they survive the suprema:
    \begin{align*}
        d_H(\mathcal{G},\,\tilde{\mathcal{G}}) &= \max \left\{\sup\nolimits_{y \in \mathcal{G}} d(y,\tilde{\mathcal{G}}),\, \sup\nolimits_{x \in \tilde{\mathcal{G}}} d(x,\, \mathcal{G}) \right\} \\
         &\leq \max\{S, S\} = S < \epsilon.
    \end{align*}

\emph{(Optimum.)} By the above, $d(x^*, \tilde{\mathcal G}) < \epsilon$, so there exists some $\tilde{x} \in \tilde{\mathcal G}$ with $\|x^* - \tilde{x}\|_1 < \epsilon$.
Optimality of $\tilde{x}^*$ and the Lipschitz bound give
    \begin{align*}
        f(x^*) - f(\tilde{x}^*) &\leq f(x^*) - f(\tilde{x}) \\
        & \leq \lambda \|x^* - \tilde{x} \|_1 \\
        & < \lambda \, \epsilon,
    \end{align*}
    and exchanging the roles of $\mathcal{G}$ and $\tilde{\mathcal{G}}$ bounds the reverse difference in the same way.
\end{proof}

The hypothesis of Proposition~\ref{prop:stability} is a statement about projections, not about the learned models.
The next result instead derives projection, set, and objective error
bounds from a uniform approximation bound on the radial inverses, with
asymptotic consistency result under continuity alone.
We define errors on radial inverses using the following extended metric on $\mathbb{R} \cup \{\infty\}$:
\[
d_\infty(x, y) =
\begin{cases}
    |x - y|, \quad & x, y \in \mathbb{R}, \\
    0, \quad &x = y = \infty, \\
    \infty, & \text{otherwise.}
\end{cases}
\]

\begin{prop}[End-to-end error propagation]
\label{prop:error_propagation}
Let $\mathcal G \subseteq [0,b]$ be a compact normal set as defined in \eqref{eq:defgh}, and let each $\phi_i: \mathbb{R}^n_+ \times \mathbb{R} \to \mathbb{R}_+ \cup \{\infty\}$, $i \in [m_g]$, satisfy properties (i)--(iii) of Proposition~\ref{prop:h_properties} and approximate $\rho_{g_i}$ to within $\epsilon \in (0,1)$,
\begin{equation}\label{eq:ri_error}
\sup_{\substack{\|x\|_1 \le \|b\|_1\\ i \in [m_g]}} d_\infty\left(\rho_{g_i}(x,u_i),\, \phi_i(x,u_i)\right) < \epsilon .
\end{equation}
Write $\tilde{\mathcal G}$ for the set whose monotone projection is the approximation \eqref{eq:proj_approx}, and let $x^*$ and $\tilde{x}^*$ maximise $f$ over $\mathcal G$ and $\tilde{\mathcal G}$ respectively.
Then $\tilde{\mathcal G}$ is a compact normal set, and the feasible region is perturbed by at most the learning error,
\begin{equation}\label{eq:error_bound}
d(\tilde{x}^*, \mathcal{G}) \;\le\; d_H(\mathcal{G}, \tilde{\mathcal{G}}) \;<\; \|b\|_1 \frac{\epsilon}{1-\epsilon}.
\end{equation}
Consequently, the optimal value is controlled.
If $f$ is increasing and continuous, then $|f(x^*) - f(\tilde{x}^*)| \to 0$ as $\epsilon \to 0$, uniformly over the approximations satisfying \eqref{eq:ri_error}, and if $f$ is $\lambda$-Lipschitz this holds at the explicit rate
\begin{equation}\label{eq:obj_error_bound}
|f(x^*) - f(\tilde{x}^*)| < \lambda \|b\|_1 \,\frac{\epsilon}{1 - \epsilon}.
\end{equation}
\end{prop}

\begin{proof}
    Define the aggregate radial inverse functions
\begin{align*}
    \rho(x) &:= \max_{i \in [m_g]} \rho_{g_i}(x,u_i), \\
    \phi(x) &:= \max_{i \in [m_g]} \phi_i(x,u_i),
\end{align*}
and write $S < \epsilon$ for the supremum in \eqref{eq:ri_error}.
Since a maximum of finitely many functions differs by at most the largest of their differences, $d_\infty(\rho(x), \phi(x)) \le S$ for all $x \in \mathbb{R}^n_+$ with $\|x\|_1 \le \|b\|_1$.

\emph{(The perturbed set.)} Since each $\phi_i$ satisfies properties (i)--(iii) in Proposition~\ref{prop:h_properties}, Proposition~\ref{prop:inverse} implies that each $\phi_i(\cdot,u_i)$ is the gauge of a normal set.
Their maximum
\(   \phi(x)=\max_{i\in[m_g]}\phi_i(x,u_i)
\)
is therefore the gauge of their intersection, which we denote by
$\tilde{\mathcal G}$.
We can assume that $\tilde{\mathcal G}$ is closed since $\gamma_S = \gamma_{\text{cl}\,S}$ for any normal set $S$.
By Proposition~\ref{prop:proj_via_ri}, the projection operators admit the representations
\[
\pi_{\mathcal G}(x) = \frac{x}{\rho(x)},
\quad
\pi_{\tilde{\mathcal{G}}}(x) = \frac{x}{\phi(x)}.
\]
We first show that $\tilde{\mathcal G}$ is bounded, and hence compact. Let $\mu = (1 - \epsilon)^{-1}$ and consider any $z \in \mathbb{R}^n_+$ with $\|z\|_1 > \mu\|b\|_1$. Define
\[
\bar{z} := \frac{\|b\|_1}{\|z\|_1} z,
\]
so that $\|\bar{z}\|_1 = \|b\|_1$. Since $\mathcal G \subseteq [0,b]$, every point in $\mathcal G$ satisfies $\|x\|_1 \le \|b\|_1$. Therefore, $\alpha \bar{z} \notin \mathcal G$ for all $\alpha > 1$, which implies $\rho(\bar z) \ge 1$ and hence $\phi(\bar z) > 1 - \epsilon$. Positive homogeneity then gives
\[
\phi(z) = \frac{\|z\|_1}{\|b\|_1}\phi(\bar{z}) > \mu(1-\epsilon) = 1.
\]
Hence $z \notin \tilde{\mathcal G}$, implying that $\tilde{\mathcal G}$ is bounded.
Thus $\tilde{\mathcal G}$ is compact, and $\tilde{x}^*$ is precisely a maximiser of $f$ over $\tilde{\mathcal G}$.

\emph{(Uniformly close projections.)}
We now bound projections onto these two sets.
For $x=0$, both projections are zero, so the bounds hold trivially.
We therefore consider $x \in\mathbb{R}^n_+ \setminus \{0\}$ and define the rescaled point
\[
\bar{x} := \frac{\|b\|_1}{\|x\|_1} x,
\]
so that $\|\bar{x}\|_1 = \|b\|_1$.
By positive homogeneity of $\rho$, we have $\rho(\bar{x}) = \frac{\|b\|_1}{\|x\|_1}\rho(x)$, which implies $\pi_{\mathcal G}(\bar{x}) = \pi_{\mathcal G}(x)$. The same holds for $\phi$, implying
$\pi_{\tilde{\mathcal G}}(x) = \pi_{\tilde{\mathcal G}}(\bar{x})$.
If either $\rho(\bar{x})$ or $\phi(\bar{x})$ is infinite then both are, as $d_\infty(\rho(\bar{x}), \phi(\bar{x})) \leq S < \infty$, hence both projections are zero and the bound holds.
We may therefore assume $\rho(\bar{x}), \phi(\bar{x}) \in \mathbb{R}_+$, so that
\begin{align*}
    \|\pi_{\mathcal G}(x) - \pi_{\tilde{\mathcal G}}(x)\|_1 &= \|\pi_{\mathcal G}(\bar{x}) - \pi_{\tilde{\mathcal G}}(\bar{x})\|_1 \\
    &= \|\bar{x}\|_1 \cdot \left| \frac{1}{\rho(\bar{x})} - \frac{1}{\phi(\bar{x})} \right| \\
    &= \|b\|_1 \cdot \frac{|\rho(\bar{x}) - \phi(\bar{x})|}{|\rho(\bar{x})\,\phi(\bar{x})|}.
\end{align*}
As above $\rho(\bar{x})\ge 1$ and $\phi(\bar{x})\ge\rho(\bar{x})-S\ge 1-S$, so the denominator is at least $1 - S$, and since the bound holds for every $x$,
\begin{equation}\label{eq:proj_error}
\begin{split}
\sup\nolimits_{x \in \mathbb{R}^n_+}\|\pi_{\mathcal G}(x) - \pi_{\tilde{\mathcal G}}(x)\|_1
&\le \|b\|_1 \frac{S}{1-S}\\
&< \|b\|_1 \frac{\epsilon}{1-\epsilon},
\end{split}
\end{equation}
the last step because $t \mapsto t/(1-t)$ is increasing on $[0,1)$ and $S < \epsilon$.
Applying Proposition~\ref{prop:stability} with this bound gives \eqref{eq:error_bound}, and gives \eqref{eq:obj_error_bound} when $f$ is $\lambda$-Lipschitz.

\emph{(Consistency.)} Now let $f$ be merely continuous, and restrict to $\epsilon \le \tfrac{1}{2}$.
The boundedness argument above then places $\mathcal G$ and every admissible $\tilde{\mathcal G}$ inside the fixed compact set $K := \{z \in \mathbb{R}^n_+ : \|z\|_1 \le 2\|b\|_1\}$, on which $f$ is uniformly continuous.
Since $f$ is increasing and $\pi_{\mathcal G}(x^*) \ge x^*$ lies in $\mathcal G$, optimality gives $f(x^*) = f(\pi_{\mathcal G}(x^*))$, and likewise $f(\tilde{x}^*) = f(\pi_{\tilde{\mathcal G}}(\tilde{x}^*))$.
Testing $x^*$ against $\pi_{\tilde{\mathcal G}}(x^*) \in \tilde{\mathcal G}$ and $\tilde{x}^*$ against $\pi_{\mathcal G}(\tilde{x}^*) \in \mathcal G$ yields
\[
f(x^*) - f(\tilde{x}^*) \le f(\pi_{\mathcal G}(x^*)) - f(\pi_{\tilde{\mathcal G}}(x^*)),
\]
\[
f(\tilde{x}^*) - f(x^*) \le f(\pi_{\tilde{\mathcal G}}(\tilde{x}^*)) - f(\pi_{\mathcal G}(\tilde{x}^*)).
\]
In each bound the two arguments of $f$ lie in $K$ and, by \eqref{eq:proj_error}, differ in norm by at most $\|b\|_1 \epsilon/(1-\epsilon)$.
Uniform continuity on $K$ therefore bounds both right-hand sides by a modulus depending on $\epsilon$ alone, which vanishes as $\epsilon \to 0$.
\end{proof}

The bounds \eqref{eq:error_bound} and \eqref{eq:obj_error_bound} scale approximately linearly in $\epsilon$ for small approximation errors, while they become increasingly conservative as $\epsilon$ approaches 1.

\section{Polyblock Outer Approximation (POA) details}\label{app:poa}

\subsection{Bisection search}\label{app:bisection}
At each POA iteration, the selected point $x\in V^t$ is a maximal
vertex of a polyblock satisfying \(P_{V^t}\supseteq\mathcal F,\)
and
$x\in\mathcal H$. For $t\geq1$ this follows from the vertex-pruning
step; for $t=0$, it follows from $\mathcal F\neq\emptyset$,
$\mathcal F\subseteq P_b$, and the upward closedness of $\mathcal H$. For $x=0$, the projection is zero by convention, so suppose $x\neq0$.
We claim that \(\pi_{\mathcal G}(x)=\alpha_{\mathcal G}(x)x\) satisfies \(\alpha_{\mathcal G}(x)\leq 1.\)

Suppose instead that $\alpha_{\mathcal G}(x)>1$.
Since $\mathcal G$ is compact, the supremum in the definition of
$\alpha_{\mathcal G}(x)$ is attained, so
$\alpha_{\mathcal G}(x)x\in\mathcal G$.
Moreover, $\mathcal H$ is upward closed, and therefore
\[
    x\in\mathcal H,
    \qquad
    \alpha_{\mathcal G}(x)x\geq x
    \quad\Longrightarrow\quad
    \alpha_{\mathcal G}(x)x\in\mathcal H.
\]
Hence
\(    \alpha_{\mathcal G}(x)x
    \in
    \mathcal G\cap\mathcal H
    =
    \mathcal F
    \subseteq
    P_{V^t}.\)
But $\alpha_{\mathcal G}(x)x$ strictly dominates $x$, contradicting
the maximality of $x$ in $P_{V^t}$ and implying \(
\alpha_{\mathcal G}(x)\leq1.
\)
Consequently, when the constraint functions are known, the projection
search may be restricted to $r\in[0,1]$. Define
\[
    q_x(r)
    :=
    \max_{i\in[m_g]}
    \bigl(g_i(rx)-u_i\bigr).
\]
Since $\mathcal G$ is non-empty and normal, $0\in\mathcal G$, so
$q_x(0)\leq0$.
If $q_x(1)\leq0$, then $x\in\mathcal G$ and
$\alpha_{\mathcal G}(x)=1$.
Otherwise, bisection on the monotone condition
$q_x(r)\leq0$ over $[0,1]$ approximates the largest feasible scaling
$\alpha_{\mathcal G}(x)$.

\subsection{Vertex refinement}
The polyblocks $V^t$ are refined by cutting out the cone $(z^t, \infty)$ at each iteration. The remaining set is a polyblock defined by the vertices

\begin{equation}\label{eq:vertex_refine}
    \hat{V}^{t+1} = (V^{t}\setminus(z^t, \infty)) \cup W^t,
\end{equation}
where
$$W^t = \{ s - (s_i - z^t_i)\,e_i : s \in V^t,\, s > z^t,\, i\in [n] \},$$
and $e_i$ is the $i$'th basis unit vector. It can be seen that $$P_{\hat{V}^{t+1}} = P_{V^t} \setminus (z^t, \infty).$$
Since $\mathcal{H}^c$ is normal, any rectangle generated by a vertex
outside $\mathcal H$ contains no feasible point. Such vertices may
therefore be discarded without removing any point of $\mathcal F$.
We additionally discard vertices dominated by another remaining
vertex, since they do not change the represented polyblock. Thus
\begin{equation} \label{eq:vertex_update}
    V^{t+1} =
    \operatorname{Max}
    \bigl(
        \hat V^{t+1}\cap\mathcal H
    \bigr),
\end{equation}
where $\operatorname{Max}(S)$ denotes the set of undominated elements of $S$.

\begin{figure*}[t]
\begin{minipage}[t]{0.48\linewidth}
\algtitle{alg:poa}{The Polyblock Outer Approximation algorithm}
\begin{algorithmic}
    \STATE \hspace{-3mm}\textbf{Inputs: $f,\, \pi_\mathcal{G},\, \mathcal{H},\, b,\, \epsilon$}
    \STATE \hspace{-3mm}\textbf{Output:} An $\epsilon$-optimal feasible solution of \eqref{eq:mono_prob}

    \STATE $V \gets \{b\}$
    \STATE $\hat{x} \gets b$
    \STATE $f_\text{best} \gets -\infty$
    \WHILE{$f_\text{best} + \epsilon < f(\hat{x})$}
        \STATE $z \gets \pi_\mathcal{G}(\hat{x})$
        \IF{$f(z) > f_\text{best}$ and $z \in \mathcal{H}$}
            \STATE $\hat{z} \gets z$
            \STATE $f_\text{best} \gets f(z)$
        \ENDIF
        \STATE $V \gets$ Update as \eqref{eq:vertex_update}
        \STATE $\hat{x} \gets \argmax_{x\in V} f(x)$
    \ENDWHILE
    \STATE \textbf{return} $\hat{z}$
\end{algorithmic}
\vspace{0.5ex}\hrule
\end{minipage}
\hfill
\begin{minipage}[t]{0.48\linewidth}
\algtitle{alg:bisection}{The bisection search algorithm used for computing monotone projections}
\begin{algorithmic}
    \STATE \hspace{-3mm}\textbf{Inputs: $\{g_i\}_{i \in [m_g]},\, x,\, \epsilon$}
    \STATE \hspace{-3mm}\textbf{Output:} Projection of $x$ onto $\mathcal G$
    \STATE $g'(x) := \max_i (g_i(x) - u_i)$
    \STATE \textbf{If} \ $g'(x) \le 0:$ \ \textbf{return} \ $x$

    \STATE $(r_{\text{max}},\, r_{\text{min}},\, r) \gets (1,\, 0,\, 0.5)$

    \WHILE{$r_\text{max} - r_\text{min} > \epsilon$}
        \IF{$g'(rx) > 0$}
            \STATE $r_{\text{max}} \gets r$
        \ELSE
            \STATE $r_{\text{min}} \gets r$
        \ENDIF
        \STATE $r \gets (r_{\text{max}} + r_{\text{min}}) / 2$
    \ENDWHILE
    \STATE \textbf{return} $r_{\min}x$
\end{algorithmic}
\vspace{0.5ex}\hrule
\end{minipage}
\end{figure*}

\subsection{Convergence}\label{app:polyblocks_solver}
The refinement continues until the stopping condition
$$\max_{s \leq t,\, z^s \in \mathcal{H}} f(z^s) + \epsilon > \max_{x \in V^t} f(x)$$
is met for a tolerance $\epsilon > 0$, at which point the best feasible candidate found so far is $\epsilon$-optimal.
Without further assumptions this condition need not be reached in finitely many iterations.
Our experiments use a monotone solver provided by the \texttt{polyblocks} package \citep{refined_poa}, which implements a series of improvements over the base POA algorithm \ref{alg:poa}.
In particular, the solver adopts a $\delta$-relaxed optimality criterion which guarantees finite termination.
For a tolerance $\delta \ge 0$, let
$$\mathcal{G}_\delta := \{x \in \mathcal{G} : x + v \in \mathcal{G}, \forall \|v\|_\infty \leq \delta\}, \qquad \mathcal{F}_\delta := \mathcal{G}_\delta \cap \mathcal{H},$$
be the $L_\infty$-erosion of the normal constraint set (equivalently, the points whose closed $\delta$-ball lies in $\mathcal{G}$), and the corresponding eroded feasible set.
A point $x^\star$ is a \emph{$\delta$-relaxed $\epsilon$-optimal} solution if it is feasible, $x^\star \in \mathcal{F}$, and $f(x^\star) + \epsilon > f(x)$ for every $x \in \mathcal{F}_\delta$; that is, it is $\epsilon$-optimal over the points that remain feasible under any $\delta$-perturbation.

\begin{theorem} \label{thm:POA}
    Suppose the feasible set is bounded away from the coordinate axes, $\mathcal{F} \subset [a,b]$ for some $a > 0$.
    Then for any $\delta, \epsilon > 0$, the \texttt{polyblocks} solver terminates in finitely many iterations, returning a $\delta$-relaxed $\epsilon$-optimal solution of \eqref{eq:mono_prob} or else proving that $\mathcal{F}_\delta = \emptyset$.
    It requires no continuity or smoothness assumption on the objective or constraint functions.
\end{theorem}

The axes-bound condition is satisfied by the offset $x_{\min}=0.1$ of Section~\ref{sec:results}.
In our experiments, the reduction in the optimal objective value resulting from the $\delta$-relaxation is typically drowned out by the objective tolerance $\epsilon$.
We refer to \citet{refined_poa} for the proof, and further implementation details.

\section{Smoothed min--max networks} \label{apen:maxmin_details}
In this appendix, we detail the power-mean smoothing of the min--max layers introduced in Section~\ref{ssec:maxmin}: its numerically stable computation, the three parameterisations used by our models, and practical considerations for sharpness and initialisation.

\subsection{Power-mean smoothing in log space}
Consider a single output feature of the layer \eqref{eq:minmax}, and suppose each affine unit has been mapped to a positive base value $u_{k,j} > 0$ (the three ways we obtain such a base are described below).
The smoothed layer output is a power-mean-of-power-means,
\begin{multline*}
\widetilde{\mu}(u) = M_{-p}\Big( M_p(u_{1,1},\dots,u_{1,J}),\, \dots,\\ M_p(u_{K,1},\dots,u_{K,J}) \Big),
\end{multline*}
with $M_p$ as in \eqref{eq:powermean}: the inner $M_p$ approximates the within-group maximum and the outer $M_{-p}$ the across-group minimum, recovering the hard min--max as $p \to \infty$. Writing $a_{k,j} = \log u_{k,j}$, power means satisfy
\begin{equation*}
\log M_p(u) = \operatorname{LSE}_p(a)
\end{equation*}
where $\operatorname{LSE}$ denotes log-sum-exp, so $\widetilde{\mu}$ is computed as a nested pair of log-sum-exp reductions in log space. This is numerically stable even at large sharpness. Since $M_p$ and $M_{-p}$ are increasing in every argument, the smoothing preserves monotonicity at every sharpness.

\subsection{Parameterisations}
The three cases used by our models differ only in how the positive base $u_{k,j}$ is obtained from the affine unit outputs $z_{k,j} = w_{k,j}^\top v + b_{k,j}$:
\begin{itemize}[leftmargin=*]
    \item \textbf{Biased (non-homogeneous).} The unit outputs $z_{k,j}$ may take any sign, and the base is $e^{z_{k,j}}$; the smoothing then reduces to ordinary log-sum-exp (soft-max/soft-min with temperature $1/p$), and the layer returns the log-space value $\log\widetilde{\mu}$, so that its output is a smoothed min--max of the $z_{k,j}$ themselves. This case is used for the gate stack $\chi_\theta$ and for the non-homogeneous ablations. For softplus-parameterised (non-negative) weights, monotonicity in $v$ is preserved since log-sum-exp is increasing.
    \item \textbf{Bias-free, sign-constrained (homogeneous and monotone).} Weights are parameterised as $w = \operatorname{softplus}(\omega) > 0$, so that $z_{k,j} = w_{k,j}^\top v > 0$ for every nonzero
    $v\in\mathbb R_+^d$ and the base is $z_{k,j}$ itself. Since power means are positively homogeneous, the smoothed output is \emph{exactly} homogeneous of degree one at any sharpness. A non-increasing direction is obtained by negating the output.
    \item \textbf{Bias-free, unconstrained sign (homogeneous only).} The unit outputs may take any sign, so there is no positive base to exploit directly, and plain log-sum-exp smoothing would be homogeneous only in the hard limit (it is shift- but not scale-equivariant). Instead, we normalise by the root-mean-square $s = \operatorname{rms}(z) $, which is itself positively homogeneous of degree one, apply the log-sum-exp smoothing to the scale-invariant quantity $z/s$, and rescale the result by $s$. The result is exactly homogeneous at any sharpness ($s = 0$ only when $z = 0$, where the output is set to zero).
\end{itemize}
In all three cases, the guarantees are architectural: they hold for every parameter setting, so no property needs to be promoted through regularisation during training or verified afterwards.

\subsection{Sharpness and initialisation}
The sharpness is parameterised as $p = p_{\min} + \operatorname{softplus}(\theta_p)$ with a learnable $\theta_p$ per layer.
The floor $p_{\min} \ge 1$ keeps power-mean gradients finite as unit outputs approach zero.
We use an initial value of $p_{\min} = 2$, but this floor can also be raised over the course of training (a sharpening schedule) to tighten the fit to the hard min--max, which we found beneficial in preliminary experiments.

Initialisation ensures unit outputs keep a consistent scale as layers are stacked: softplus-parameterised (positive) weights add coherently and are initialised around $1/d$ for $d$ inputs, while unconstrained zero-mean weights add incoherently and are initialised with standard deviation $1/\sqrt{d}$.

\section{Additional details on HM-RI models}
\label{app:hmri_details}

This section provides additional details and design considerations for the HM-RI architecture introduced in Section~\ref{ssec:hm-ri}. In particular, we elaborate on the choice of network parameterisation and its practical implications.

Both subnetworks $\psi_{\theta}$ and $\chi_\theta$ in \eqref{eq:subnets} are implemented as stacks of smoothed min--max layers (Appendix~\ref{apen:maxmin_details}), each with $K$ groups of $J$ affine units per output feature and $C$ output features per layer. The number of gated features $C$ combined by the minimum is kept separate from the per-layer group count $K$, since the construction behind Proposition~\ref{prop:universal} is a minimum over a grid of points and its resolution is therefore governed by $C$.

\paragraph{Conditioning on instance parameters.} For the parameterised families of Section~\ref{sec:results}, the instance parameters $z$ enter in two places. As in the unconditional model, a small unconstrained MLP maps $z$ to an additive shift of the first-layer output of $\chi_\theta$, which lets $z$ set the gate of each channel. On its own, however, this is a substantive restriction: $\psi_\theta$ would then be shared by every instance of the family, so each instance could only \emph{reweight} one fixed dictionary of $C$ shape functions. Families whose level sets change shape with $z$ cannot be represented that way, and we found this to be the binding constraint in practice rather than data or capacity -- with $z$ entering only through the gate, accuracy on the quadratic family was unchanged by moving from $512$ to unlimited training samples, by five-fold increases in $C$, or by five-fold longer training.

We therefore additionally let $z$ rescale the \emph{input} of $\psi_\theta$, coordinatewise,
\begin{equation}\label{eq:xmod}
\psi_\theta\big(x \odot \operatorname{softplus}(\omega_\theta(z))\big),
\end{equation}
for a second unconstrained MLP $\omega_\theta$. Because the scaling is strictly positive and acts coordinatewise, it is exactly positively homogeneous of degree one in $x$ and order-preserving, so every structural property of Section~\ref{ssec:hm-ri} continues to hold exactly; we verify this numerically alongside the ablations of Appendix~\ref{apen:ablations}. Empirically this is worth roughly ten points of projected objective on the quadratic family and reduces the interquartile spread of the predicted-to-true radial scale by a factor of four (Appendix~\ref{apen:results}). Replacing the coordinatewise scaling by a full non-negative matrix $x \mapsto x A(z)$, which preserves the same properties, performs \emph{worse}: the additional parameters cost more in optimisation than the extra flexibility gains.

\subsection{Extended-valued constraint functions}
\label{app:extended_valued}

As discussed in Section~\ref{ssec:mono_opt_prelim}, we allow the constraint functions $g_i$ to take values in the extended real line. This costs nothing operationally: the projection operator $\pi_{\mathcal G}$ depends only on the geometry of the constraint set $\mathcal G$, and not on the specific parametrisation of the constraint functions $g_i$. Even if the $g_i$ take values in $\{-\infty, \infty\}$, the projection $\pi_{\mathcal G}$ remains well-defined and finite provided that $\mathcal G$ is non-empty and compact. Under this assumption, the scaling constant $\alpha_{\mathcal G}(x)$ exists and is finite, ensuring that the projection is properly defined.

The benefit is that Proposition~\ref{prop:inverse} then requires no conditions beyond the structural properties (i)--(iii) that the HM-RI architecture targets directly. A learned model satisfying (i)--(iii) exactly is guaranteed to be the radial inverse of some increasing $\tilde g_i$, and hence to define a genuine instance of \eqref{eq:mono_prob}, with no auxiliary condition that must be separately promoted during training or checked after it.

\section{Experimental details}\label{apen:exp_details}

\subsection{Hardware}
All experiments were conducted on a single machine with an AMD Ryzen 5 3600 CPU (6 cores / 12 threads).
Both training and solving were run on the CPU: the models are small enough (roughly $7\times10^4$ parameters) that per-operation dispatch overhead, rather than arithmetic throughput, dominates at these batch sizes, so a GPU gives no benefit at the batch sizes used here.

\subsection{POA implementation}\label{apen:poa_implement}
We use the tree-based POA solver of the \texttt{polyblocks} package \citep{refined_poa} with projections anchored at the origin. Each solve is run to a relative objective tolerance of $2\%$, an iteration limit of 200, and a time limit of 30\,s.
Where bisection is used (for the M-Net surrogate, the \textsc{Oracle} method, and for computing the projected-objective metric), we use a tolerance of $10^{-4}$.

\paragraph{Empty-surrogate correction.} Let $x_l:=x_{\min}\mathbf{1}$ denote the lower corner of the optimisation box. A learned constraint set can exclude the box lower bound $x_l$, or be conservative enough that POA records no incumbent; in either case the solve returns nothing.
When this occurs, we record $x_l$, which is feasible for the true constraints by construction, as the returned solution.
This avoids discarding failed solves and scores them conservatively at a known feasible point; such failures become more frequent for M-Net as dimension increases.

\subsection{Model configuration}

All radial-inverse models used a consistent architecture across problem families: two smoothed min--max layers per stack, with $K = 6$ groups of $J = 4$ affine units and $C = 6$ output features per layer, and a sharpness floor of $p_{\min} = 2$. The instance parameters $z$ enter through an unconstrained MLP with two hidden layers of 64 units and GELU activations.

The direct constraint surrogate (M-Net) uses the same min--max configuration with a scalar output, conditioned on $z$ in the same way, so that all learned methods have comparable capacity. The local baselines (SLSQP and COBYLA) optimise the known objective subject to the M-Net surrogate constraint.

\subsection{Training}
All models were trained with Adam at a learning rate of $3\times10^{-3}$, a batch size of 256, and 2000 training steps, drawing minibatches from the fixed training buffer of the corresponding data regime. The scale range for the scale-augmented objective of Section~\ref{ssec:simplified} was $t \sim \mathcal{U}[1, 2]$, chosen to
cover outward radial scales frequently encountered in POA queries.

Rather than fixing the smoothing exponent of Section~\ref{ssec:maxmin} at its floor $p_{\min}$, we raise the floor geometrically from $2$ to $32$ over the course of training, so that early steps optimise a smooth surrogate and later steps a tight approximation of the true min--max. The same schedule is applied to \emph{every} learned model, including the direct surrogate M-Net, so that no method is advantaged by it.

\subsection{Problem generation}\label{app:problem_generation}
For each family in Section~\ref{sec:results}, evaluation instances are defined over the box $[0.1, 1]^n$. Constraint and objective parameters are sampled as described in Appendix~\ref{sec:details_applications}. For each of the 50 evaluation instances, the threshold of constraint $i$ is set to $u_i = g_{z_i}\big(\alpha_i (x_{\max} - x_{\min})\big)$ with $\alpha_i \sim \mathcal{U}[0.2, 0.8]$.
This gives a reference scale in
$[0.18, 0.72]$ with median $0.45$, and pins how tightly each constraint binds independently of shifts in a family's value distribution across dimensions.
Since $x_{\min} = 0.1$ lies below the lower end of this range, $x_l$ is guaranteed to remain feasible for any number of constraints; none of the $1200$ generated instances we checked was infeasible at $x_l$. Instance seeds are disjoint from training seeds. Training inputs are sampled as $\tilde{x}=u x$, where $u\sim\mathcal{U}[0.5,1]$ and each $x_i\sim\mathcal{U}[0.1,1]$ independently, with the range selected on separate validation instances.

\section{Additional results}\label{apen:results}
This appendix provides additional quantitative results complementing the main experiments in Section~\ref{sec:results}.

\subsection{Transmit power optimisation}\label{apen:tpower}
The problems considered in Section~\ref{sec:results} are all abstract.
We provide a more practical example in cellular networks, involving the minimisation of total downlink transmit power subject to a per-user capacity requirement,
\begin{align} \label{eq:tpower}
\begin{split}
    \min_{p \in [0,1]^n} \quad & \sum\nolimits_{i=1}^n p_i \\
    \text{s.t.} \quad & c(p, s_j) \geq c_j,\quad j \in [m_g],
\end{split}
\end{align}
with $p_i$ the transmit power of cell $i$, $s_j$ the location of user $j$, and $c$ the achievable capacity.
Capacity is non-decreasing in $p$, so replacing the optimisation variable with $x = 1-p$ puts \eqref{eq:tpower} in canonical form; since
$\sum_i p_i=n-\sum_i x_i$, minimising total transmit power is equivalent
to maximising $\sum_i x_i$.

\paragraph{Setup.}
We simulate $n = 7$ cells in a hexagonal layout with the Cellular Radio Reference Model (CRRM) simulator \citep{keith-CRRM} under the standards-compliant Urban Macrocell path-loss model \citep{3gpp_channel_model}.
A user $j$ with channel gains $h_j$ is served by its best cell, $\gamma_j(p) = \max_i h_{ji} p_i / \sigma^2$, where inter-cell interference is not modelled to keep the problem monotone.
The corresponding signal strength maps to a discrete set of modulation and coding schemes and then to a spectral efficiency $\mathrm{SE}_j$, resulting in a piecewise-constant capacity function.
We take as constraint the capacity deficit relative to full power,
\begin{align*}
    g_z(x) = \mathrm{SE}_j(\mathbf{1}) - \mathrm{SE}_j(\mathbf{1}-x) \;\leq\; u_j := \mathrm{SE}_j(\mathbf{1}) - c_j ,
\end{align*}
which is non-negative and increasing.
Users are retained only if their full-power signal strength lies in $[2,14]$~dB and each threshold is drawn above the point at which its own deficit becomes positive, both to ensure the capacity function takes on diverse values.
The function's wide plateaus also call for a heavier one-sided loss weight than the quantised family, $\beta = 20$.
The protocol is otherwise that of Section~\ref{sec:results} in the \emph{Limited} regime.

\paragraph{Results.}
This problem has a unique structure compared to the previous classes, as each individual normal constraint set is a \emph{union} of half-spaces and hence the radial inverse is a minimum, $\rho_{g_z}(x,u) = \min_i x_i/c_i(u,z)$.
Table~\ref{tab:tpower} reports projected objectives, where radial-inverse methods attain $74$--$79\%$ of the oracle objective against $41.0\%$ for the direct surrogate and $39.8\%$ for the best local solver.
Paired over instances, HM-RI leads M-Net by $+34.4$ points (95\% CI $[+30.1,+38.5]$) and SLSQP by $+35.6$ $[+32.1,+38.8]$.
In physical terms, the radial-inverse methods nearly meet the capacity requirement using $2.5$--$2.8$ units of power against $4.8$ for the local solvers and $1.35$ for the \textsc{Oracle}, at a total violation of $0.05$--$0.23$ against $1.08$--$1.11$.
Every method that relies on the learned constraint rather than its radial inverse falls \emph{below} the \textsc{Random} baseline: the staircase leaves finite differences uninformative, and as the surrogate declares the central starting point infeasible they retreat to high power rather than search.
The RI variants themselves are close, M-RI leading HM-RI by $3.9$ $[+2.5,+5.2]$, and that ordering is not stable in the training budget, so the separation from the constraint-surrogate methods is the robust finding rather than the ranking within them.

\begin{table*}[t]
    \centering
    \caption{
    Transmit power optimisation ($n=7$ cells, \emph{Limited} data regime, 50 instances, three seeds).
    P-Obj.\ is the projected objective as a percentage of the \textsc{Oracle}'s; \emph{Power} is the corresponding total transmit power $\sum_i p_i$ (lower is better), for which the \textsc{Oracle} attains $1.35$.
    \emph{Random} is a uniformly random point, included as a reference baseline; both local solvers and the direct surrogate score below it.
    }
    \label{tab:tpower}
    \input{results_tpower}
\end{table*}

\subsection{Solver runtimes}\label{apen:runtimes}
Table~\ref{tab:runtimes} reports runtimes for each baseline considered in the limited-data regime.
The comparison of interest is HM-RI against M-Net, which share an identical solver and differ only in the projection primitive: a single batched forward pass versus a bisection loop of surrogate evaluations.
The speed-up tracks how many bisection steps the surrogate constraint costs, and is largest on the step and quadratic families.
The local solvers are faster in wall-clock terms but provide no global optimality guarantee.

\begin{table*}[t]
    \centering
    \caption{
    Mean wall-clock time per solve, in seconds, for each method and problem family at $n=4$ in the limited-data regime, averaged over the 50 shared instances and three training seeds.
    }
    \label{tab:runtimes}
    \input{results_runtimes}
\end{table*}

\subsection{Robustness to noisy constraint observations}\label{apen:noise}

The training pairs $(x, y)$ assume $g_z(x) = y$ exactly, whereas in most data-driven settings $y$ is a noisy \emph{observation}.
To simulate this, we corrupt each training output $(x, y) \in \mathcal{D}_i$ multiplicatively, $y \gets y\, \exp(\sigma \varepsilon)$ with $\varepsilon \sim \mathcal{N}(0,1)$, on multiplicative problems at $n = 4$ in the limited-data regime.

\begin{table*}[t]
    \centering
    \caption{
    Projected objective as a percentage of \textsc{Oracle}'s under noisy training data, multiplicative problems at $n=4$ with $512$ samples.
    $\sigma$ is the standard deviation of the log-normal multiplicative noise.
    The \emph{lost} row is the drop from the noiseless model to $\sigma = 0.8$, in points of oracle objective.
    Readings are means over ten training seeds with standard errors.
    The \textsc{Random} reference for this setting is $21.7\%$.
    }
    \label{tab:noise}
    \input{results_noise}
\end{table*}

Table~\ref{tab:noise} reveals an advantage of the homogeneity guarantee not visible in the noiseless case.
HM-RI remains the top performer at every noise level and degrades gracefully, whereas the non-homogeneous M-RI suffers twice its objective reduction, its relative homogeneity violation growing from $0.13$ to $0.46$ at $\sigma = 0.8$, while the homogeneous variants hold it at zero by construction.

\subsection{Statistical significance of method differences}\label{apen:significance}

The uncertainties in Figure~\ref{fig:results_summary} are standard errors over training seeds on a \emph{fixed} instance set, so they measure training stability rather than the reliability of a method comparison.
Because all methods solve the same instances, the right test is paired.
Table~\ref{tab:significance} reports the mean per-instance difference between HM-RI and each competitor, marking those whose $95\%$ confidence interval excludes $0$.

\begin{table*}[t]
    \centering
    \caption{HM-RI minus each method, paired over the 50 shared instances after averaging across training seeds, in percentage points relative to the oracle objective.
    Entries marked $^\ast$ have a $95\%$ bootstrap confidence interval which excludes zero.}
    \label{tab:significance}
    \input{results_significance}
\end{table*}

The advantage over the direct surrogate and the local baselines is significant in every cell except random monotone networks.
That family leaves little room to separate methods as a uniformly random point already reaches about two thirds of the oracle, so HM-RI's lead there, $0.3$ to $3.4$ points, is not resolvable at this sample size.
Within the radial-inverse variants the comparison is regime-dependent, HM-RI is significantly ahead in more than half of the cells and only significantly behind in one, its performance advantage fading with more training data.
The empirical performance gap narrows with additional noiseless data, although only the structured variants guarantee the required properties globally.

\subsection{Structural-property ablations}\label{apen:ablations}

To isolate the effect of each structural property, we train radial-inverse models for every combination of the two architectural switches of Section~\ref{ssec:simplified} (homogeneity and monotonicity).
For each variant we report the held-out regression MSE; the \emph{gauge reconstruction error}, the relative mean absolute difference between the target level $y$ and $g_z(v / \phi_\theta(v, y, z))$, measuring how accurately the model lands back on the intended level set; and a violation measure for each dropped property.

\begin{table*}[t]
    \centering
    \caption{Structural-property ablations at $n=4$, averaged over three training repetitions. Unlike Figure~\ref{fig:results_summary}, these models are trained on freshly drawn samples rather than a finite buffer, since the question here is what each property costs and guarantees rather than sample efficiency. Violation columns measure the property that the corresponding variant drops; enforced properties hold exactly, up to floating-point precision. Note that monotonicity is very nearly satisfied even by the variants that do not enforce it, whereas homogeneity is not: see the discussion below.}
    \label{tab:ablations}
    \begin{tabular}{l l | c c c c c} \toprule
Family & Method & MSE & Recon.\ err. & $x$-viol. & $y$-viol. & Hom.\ err. \\ \midrule
\multirow{4}*{Quadratic}
 & HM-RI & 0.0016 & 0.030 & 0 & 0 & $2\!\cdot\!10^{-7}$ \\
 & M-RI & 0.0018 & 0.030 & 0 & 0 & $2\!\cdot\!10^{-1}$ \\
 & H-RI & 0.0020 & 0.037 & $2\!\cdot\!10^{-6}$ & 0 & $1\!\cdot\!10^{-7}$ \\
 & RI & 0.0019 & 0.030 & $4\!\cdot\!10^{-6}$ & 0 & $2\!\cdot\!10^{-1}$ \\
\midrule
\multirow{4}*{Multiplicative}
 & HM-RI & 0.0113 & 0.120 & 0 & 0 & $2\!\cdot\!10^{-7}$ \\
 & M-RI & 0.0122 & 0.121 & 0 & 0 & $8\!\cdot\!10^{-1}$ \\
 & H-RI & 0.0108 & 0.115 & 0 & 0 & $1\!\cdot\!10^{-7}$ \\
 & RI & 0.0069 & 0.117 & $2\!\cdot\!10^{-5}$ & 0 & $6\!\cdot\!10^{-1}$ \\
\midrule
\multirow{4}*{Monotone nets}
 & HM-RI & 0.0340 & 0.110 & 0 & 0 & $2\!\cdot\!10^{-7}$ \\
 & M-RI & 0.0214 & 0.092 & 0 & 0 & $8\!\cdot\!10^{-1}$ \\
 & H-RI & 0.0311 & 0.113 & $5\!\cdot\!10^{-5}$ & 0 & $1\!\cdot\!10^{-7}$ \\
 & RI & 0.0207 & 0.094 & $3\!\cdot\!10^{-5}$ & 0 & $8\!\cdot\!10^{-1}$ \\
\midrule
\multirow{4}*{Step}
 & HM-RI & 0.0030 & 0.028 & 0 & 0 & $2\!\cdot\!10^{-7}$ \\
 & M-RI & 0.0032 & 0.028 & 0 & 0 & $2\!\cdot\!10^{-1}$ \\
 & H-RI & 0.0074 & 0.056 & $1\!\cdot\!10^{-4}$ & $4\!\cdot\!10^{-6}$ & $1\!\cdot\!10^{-7}$ \\
 & RI & 0.0033 & 0.038 & $2\!\cdot\!10^{-6}$ & 0 & $2\!\cdot\!10^{-1}$ \\
\bottomrule
\end{tabular}
\end{table*}

These results empirically confirm that the architectural restrictions do indeed enforce the required properties.
However, we see that the monotonicity-violation columns are near zero for \emph{every} variant, including those that do not enforce it.
The gauge regression target is itself monotone, so gradient training recovers an approximately monotone function from data alone.
The value of enforcing monotonicity architecturally is therefore not that unconstrained models are grossly non-monotone on the training distribution, but that the guarantee holds everywhere and at every parameter setting, including the off-distribution queries POA generates during a solve.
The dimension sweep of Table~\ref{tab:dim_study} highlights a larger performance gap in the harder, higher-dimensional settings.

Note also that the fitting cost of the constraints is family-dependent. On the quadratic family all four variants fit almost equally well ($0.0016$--$0.0020$), so structure costs essentially nothing there; on the monotone-network family the homogeneous variants pay a clear premium ($0.034$ against $0.021$), which is the expressiveness price of the constraint.
This fitting penalty does not translate directly into poorer solver performance, which is the point of Appendix~\ref{app:training_data_analysis}.

The same reading explains the pattern in Figure~\ref{fig:results_summary} and the dimension sweep of Table~\ref{tab:dim_study}: a model that satisfies one property exactly and the other only approximately performs well whenever the approximation happens to suffice, and poorly when it does not, while HM-RI stays near the top throughout.

\subsection{Scaling with problem dimension}\label{sec:higher_dim}
The main text fixes $n=4$.
To explore how the method scales, we sweep the dimension on the quadratic family in the large-data regime (50 instances per dimension, three training seeds).
We report a single family so that the metric stays comparable across $n$: the quadratic objective is degree two at every dimension, so a fixed directional error costs the same fraction of the objective at any size.
The learned methods keep the paper's settings, except at $n=7$, where the number of training steps is increased from $2000$ to $6000$ for each model to ensure they each saturate their corresponding losses.
The \textsc{Oracle} is given an uncapped iteration limit and a ten-minute budget to accommodate the large problem size.

\begin{table*}[t]
    \centering
    \caption{
    Solution quality and cost against problem dimension, quadratic family, large-data regime (50 instances per dimension, three training seeds).
    P-Obj.\ is the projected objective as a percentage of the \textsc{Oracle}'s, scored per instance and then averaged; uncertainties are standard errors over instances, which dominate those over seeds several-fold.
    Solve times are means over instances and seeds.
    \emph{Random} solves nothing, so no time is reported for it.
    The \textsc{Oracle} successfully solves all instances up to $n=6$ and all but six at $n=7$, which were excluded from all comparisons.
    }
    \label{tab:dim_study}
    \input{results_dim_study}
\end{table*}

Table~\ref{tab:dim_study} reports projected objectives along with runtime for each method.
Solution quality is seen to degrade gently while the methods' ordering remains mostly unchanged which implies they all scale similarly across $n$.
This decline appears mostly due to the metric rather than to any method, as \emph{Random}'s objective value falls by a comparable amount over the same range.

These results also further support the runtime advantage of radial-inverse variants as their runtime grows roughly tenfold between $n=4$ and $n=7$ (0.19 to 2.07\,s for HM-RI) while \textsc{Oracle} and M-Net scale far worse.
In particular, \textsc{Oracle}'s runtime grows to over $70$ times that of the RI methods at $n=7$ ($149$ against $2.07$\,s), showing that despite relying on the same underlying solver, RI approximation methods seem more scalable than their exact counterparts.

\subsection{Training data analysis}\label{app:training_data_analysis}
The two data regimes in Figure~\ref{fig:results_summary} (512 and 8192 training samples) show the transition from a structure-dominated to a data-dominated regime.
To characterise this transition, we additionally report held-out test losses of the radial-inverse variants under both regimes.

\begin{table*}[t]
    \centering
    \caption{Held-out test loss (mean squared error of the scale regression) for radial-inverse variants under both data regimes, at $n=4$, measured on freshly drawn samples rather than the finite training buffer. Lowest loss per column in bold.}
    \label{tab:loss_errors}
    \begin{tabular}{l|c c|c c|c c|c c} \toprule
\multirow{2}*{Method} & \multicolumn{2}{c|}{Quadratic} & \multicolumn{2}{c|}{Multiplicative} & \multicolumn{2}{c|}{Monotone nets} & \multicolumn{2}{c}{Step} \\
 & 512 & 8192 & 512 & 8192 & 512 & 8192 & 512 & 8192 \\ \midrule
HM-RI & \textbf{0.0028} & \textbf{0.0013} & \textbf{0.0153} & 0.0110 & 0.0917 & 0.0385 & 0.0055 & 0.0036 \\
M-RI & 0.0029 & 0.0015 & 0.0166 & 0.0129 & 0.0925 & \textbf{0.0217} & \textbf{0.0051} & \textbf{0.0034} \\
H-RI & 0.0096 & 0.0021 & 0.0495 & 0.0101 & 0.0737 & 0.0363 & 0.0315 & 0.0041 \\
RI & 0.0350 & 0.0024 & 0.0399 & \textbf{0.0087} & \textbf{0.0647} & 0.0223 & 0.0279 & 0.0037 \\
\bottomrule
\end{tabular}
\end{table*}

The table shows that fit quality and solver quality are close to unrelated. On the monotone-network and step families a variant that drops a structural property attains the lowest loss while HM-RI wins by projected objective; on the quadratic family the relationship is not inverted, and HM-RI is both the best fit and the best solver. The loss carries no reliable signal about the property the solver depends on -- whether the learned map is the radial inverse of \emph{some} increasing function -- and a practitioner selecting a variant on held-out loss alone would not reliably select the best optimiser.

The losses do fall with additional data in every setting, so the regression problem is being solved better even when solver quality does not improve correspondingly. What the loss misses is visible directly in the constraint violations of Figure~\ref{fig:tradeoff}, where the ordering follows the structural properties rather than the loss: a model can reduce its mean squared error while drifting away from being a valid radial inverse, and its errors then act on the geometry of the feasible set rather than as bounded perturbations of it.

\end{document}

%% file: results_tpower.tex
\begin{tabular}{l l|c c} \toprule
        Type & Method & P-Obj. (\% of Oracle) & Power \\ \midrule
        \multirow{5}*{Monotone} & HM-RI & 75.4\,\tiny{$\pm$ 0.8} & 2.75 \\
         & M-RI & \textbf{79.2\,\tiny{$\pm$ 0.5}} & 2.53 \\
         & H-RI & 74.9\,\tiny{$\pm$ 0.5} & 2.77 \\
         & RI & 74.0\,\tiny{$\pm$ 1.9} & 2.82 \\
         & M-Net & 41.0\,\tiny{$\pm$ 3.0} & 4.69 \\
        \midrule
        \multirow{2}*{Local} & SLSQP & 39.8\,\tiny{$\pm$ 0.8} & 4.75 \\
         & COBYLA & 37.5\,\tiny{$\pm$ 1.8} & 4.88 \\
        \midrule
        \multicolumn{2}{l|}{Random} & 50.2\,\tiny{$\pm$ 0.8} & 4.16 \\
        \bottomrule
\end{tabular}

%% file: results_runtimes.tex
\begin{tabular}[t]{l l|c c c c } \toprule
Type & Method & Quadratic & Multiplicative & Monotone nets & Step \\ \midrule
\multirow{5}*{POA}
 & HM-RI & 0.20 & 0.41 & 0.38 & 0.15 \\
 & M-RI & 0.20 & 0.25 & 0.35 & 0.15 \\
 & H-RI & 0.16 & 0.09 & 0.05 & 0.08 \\
 & RI & 0.07 & 0.07 & 0.08 & 0.07 \\
 & M-Net & 0.81 & 1.30 & 1.12 & 0.68 \\
\midrule
\multirow{2}*{Local}
 & SLSQP & 0.03 & 0.02 & 0.05 & 0.36 \\
 & COBYLA & 0.07 & 0.11 & 0.06 & 0.05 \\
\midrule
\textsc{Oracle} & & 0.12 & 0.13 & 0.19 & 0.10 \\
\bottomrule
\end{tabular}

%% file: results_noise.tex
\begin{tabular}{r|c c c c} \toprule
$\sigma$ & HM-RI & M-RI & H-RI & RI \\ \midrule
$0$ & $72.6$\,\tiny{$\pm$ 0.5} & $68.2$\,\tiny{$\pm$ 0.8} & $61.1$\,\tiny{$\pm$ 1.5} & $65.2$\,\tiny{$\pm$ 1.2} \\
$0.05$ & $72.8$\,\tiny{$\pm$ 0.6} & $68.2$\,\tiny{$\pm$ 0.8} & $60.4$\,\tiny{$\pm$ 1.3} & $61.1$\,\tiny{$\pm$ 1.1} \\
$0.1$ & $72.9$\,\tiny{$\pm$ 0.6} & $67.1$\,\tiny{$\pm$ 0.6} & $59.7$\,\tiny{$\pm$ 1.1} & $61.1$\,\tiny{$\pm$ 1.2} \\
$0.2$ & $71.7$\,\tiny{$\pm$ 0.6} & $64.0$\,\tiny{$\pm$ 0.4} & $58.4$\,\tiny{$\pm$ 1.5} & $60.7$\,\tiny{$\pm$ 1.2} \\
$0.4$ & $68.2$\,\tiny{$\pm$ 0.5} & $53.8$\,\tiny{$\pm$ 0.7} & $57.0$\,\tiny{$\pm$ 1.1} & $59.1$\,\tiny{$\pm$ 0.9} \\
$0.8$ & $56.3$\,\tiny{$\pm$ 0.9} & $36.7$\,\tiny{$\pm$ 0.6} & $54.6$\,\tiny{$\pm$ 1.3} & $53.5$\,\tiny{$\pm$ 1.1} \\
\midrule
\multicolumn{1}{l|}{\emph{lost}} & $16.3$ & $31.5$ & $6.6$ & $11.7$ \\
\bottomrule
\end{tabular}

%% file: results_significance.tex
\begin{tabular}{l r|c c c c c c} \toprule
Family & $N$ & M-RI & H-RI & RI & M-Net & SLSQP & COBYLA \\ \midrule
Quadratic & 512 & $+0.5$ & $+6.1$$^\ast$ & $+18.5$$^\ast$ & $+23.3$$^\ast$ & $+24.0$$^\ast$ & $+23.9$$^\ast$ \\
 & 8192 & $-0.3$ & $+1.2$$^\ast$ & $+0.7$ & $+23.4$$^\ast$ & $+23.3$$^\ast$ & $+23.4$$^\ast$ \\
\midrule
Multiplicative & 512 & $+6.5$$^\ast$ & $+9.7$$^\ast$ & $+11.2$$^\ast$ & $+13.6$$^\ast$ & $+14.1$$^\ast$ & $+13.9$$^\ast$ \\
 & 8192 & $+9.8$$^\ast$ & $-4.1$$^\ast$ & $-0.1$ & $+16.4$$^\ast$ & $+16.8$$^\ast$ & $+16.9$$^\ast$ \\
\midrule
Monotone nets & 512 & $-0.7$ & $+8.9$$^\ast$ & $+1.2$ & $+0.3$ & $+2.3$ & $+1.9$ \\
 & 8192 & $+10.8$$^\ast$ & $+9.0$$^\ast$ & $+9.2$$^\ast$ & $+2.1$ & $+2.3$ & $+3.4$ \\
\midrule
Step & 512 & $+0.6$ & $+14.1$$^\ast$ & $+15.0$$^\ast$ & $+22.6$$^\ast$ & $+58.6$$^\ast$ & $+23.6$$^\ast$ \\
 & 8192 & $-0.3$ & $+0.6$ & $+1.2$$^\ast$ & $+22.2$$^\ast$ & $+57.9$$^\ast$ & $+23.3$$^\ast$ \\
\bottomrule
\end{tabular}

%% file: results_dim_study.tex
\begin{tabular}{l|cccc|cccc} \toprule
        & \multicolumn{4}{c|}{P-Obj.\ (\% of \textsc{Oracle})} & \multicolumn{4}{c}{Solve time (s)} \\
        Method & $n=4$ & $n=5$ & $n=6$ & $n=7$ & $n=4$ & $n=5$ & $n=6$ & $n=7$ \\ \midrule
        HM-RI & $96.4$\,\tiny{$\pm$ 0.6} & $96.5$\,\tiny{$\pm$ 0.5} & $91.7$\,\tiny{$\pm$ 0.9} & $85.5$\,\tiny{$\pm$ 0.9} & $0.19$ & $0.35$ & $0.52$ & $2.07$ \\
        M-RI & $96.3$\,\tiny{$\pm$ 0.5} & $95.9$\,\tiny{$\pm$ 0.5} & $90.7$\,\tiny{$\pm$ 0.9} & $86.3$\,\tiny{$\pm$ 0.9} & $0.20$ & $0.36$ & $0.54$ & $2.25$ \\
        H-RI & $95.4$\,\tiny{$\pm$ 0.6} & $94.4$\,\tiny{$\pm$ 0.6} & $89.4$\,\tiny{$\pm$ 0.9} & $84.3$\,\tiny{$\pm$ 1.0} & $0.18$ & $0.34$ & $0.52$ & $2.43$ \\
        RI & $95.4$\,\tiny{$\pm$ 0.5} & $94.7$\,\tiny{$\pm$ 0.6} & $86.6$\,\tiny{$\pm$ 0.9} & $82.9$\,\tiny{$\pm$ 1.1} & $0.18$ & $0.33$ & $0.50$ & $2.24$ \\
        M-Net & $71.9$\,\tiny{$\pm$ 2.9} & $70.4$\,\tiny{$\pm$ 2.9} & $62.6$\,\tiny{$\pm$ 2.9} & $63.3$\,\tiny{$\pm$ 2.5} & $0.81$ & $1.29$ & $1.54$ & $4.88$ \\
        SLSQP & $72.2$\,\tiny{$\pm$ 2.9} & $70.4$\,\tiny{$\pm$ 3.0} & $64.3$\,\tiny{$\pm$ 3.4} & $66.8$\,\tiny{$\pm$ 3.0} & $0.03$ & $0.04$ & $0.05$ & $0.08$ \\
        COBYLA & $72.2$\,\tiny{$\pm$ 2.9} & $69.7$\,\tiny{$\pm$ 3.1} & $64.3$\,\tiny{$\pm$ 3.4} & $67.0$\,\tiny{$\pm$ 2.9} & $0.07$ & $0.09$ & $0.12$ & $0.22$ \\
        Random & $42.6$\,\tiny{$\pm$ 2.9} & $38.0$\,\tiny{$\pm$ 2.3} & $31.6$\,\tiny{$\pm$ 1.9} & $26.3$\,\tiny{$\pm$ 1.6} & -- & -- & -- & -- \\
        \midrule
        Oracle & $100.0$\,\tiny{$\pm$ 0.0} & $100.0$\,\tiny{$\pm$ 0.0} & $100.0$\,\tiny{$\pm$ 0.0} & $100.0$\,\tiny{$\pm$ 0.0} & $0.21$ & $1.95$ & $5.19$ & $149$ \\
        \bottomrule
\end{tabular}